\documentclass{article}
\usepackage{iclr2027_conference,times}
\usepackage[utf8]{inputenc}
\usepackage[T1]{fontenc}
\usepackage[unicode]{hyperref}
\usepackage{url}
\usepackage{amsmath,amssymb,amsthm}
\usepackage{booktabs}
\usepackage{graphicx}
\usepackage{longtable}
\usepackage{natbib}

\usepackage{microtype}
\usepackage{xcolor}

\theoremstyle{plain}
\newtheorem{theorem}{Theorem}
\newtheorem{proposition}[theorem]{Proposition}
\newtheorem{lemma}[theorem]{Lemma}
\newtheorem{corollary}[theorem]{Corollary}
\theoremstyle{definition}
\newtheorem{definition}[theorem]{Definition}
\theoremstyle{remark}
\newtheorem{remark}[theorem]{Remark}

\newcommand{\R}{\mathbb{R}}

\newcommand{\E}{\mathbb{E}}
\newcommand{\Od}{O(d)}

\newcommand{\ip}[2]{\langle #1,\, #2\rangle}
\newcommand{\norm}[1]{\lVert #1\rVert}
\newcommand{\spn}{\operatorname{span}}
\DeclareMathOperator{\Hom}{Hom}
\DeclareMathOperator{\tr}{tr}
\DeclareMathOperator{\rank}{rank}
\DeclareMathOperator{\Fix}{Fix}
\DeclareMathOperator{\Stab}{Stab}
\DeclareMathOperator{\TV}{TV}
\newcommand{\dmin}{d_{\min}}
\newcommand{\dcov}{d_{\mathrm{cov}}}
\newcommand{\dproj}{d_{\mathrm{proj}}}

\title{Relational Composition as
\\
Parallel Transport:\\
What Per-Entity Frames Force in Attention}
\iclrfinalcopy % preprint mode: prints the author block, removes the review line numbers; the running header is set to "Preprint" in the style file
\author{Aakash Sarkar \\
University of California, San Francisco \\
\texttt{aakash.sarkar@ucsf.edu} % fill in
\And
Ananya Mukherjee \\
Brigham and Women's Hospital \\
\texttt{amukherjee8@bwh.harvard.edu} % fill in
\And
Motoshige Sato \\
University of California, San Francisco \\
\texttt{motoshige.sato@ucsf.edu} % fill in
}

\hypersetup{hidelinks,pdfauthor={Aakash Sarkar, Ananya Mukherjee, Motoshige Sato},pdftitle={Relational Composition as Parallel Transport: What Per-Entity Frames Force in Attention},pdfsubject={Preprint}}
\begin{document}
\maketitle

\begin{abstract}
A dot product means something only when both vectors share a coordinate system. Content words have one---the corpus fixes what every direction means---but a fresh entity, a name in a story or a variable in a program, does not: rotating its coordinates alone changes nothing. We take this per-entity frame freedom as a symmetry of the representation and ask what it forces in attention. Every invariant score depends on two entities only through their norms and the transported inner product $\ip{x_i}{T_{ij}x_j}$, and every equivariant message lies in the span of $x_i$ and $T_{ij}x_j$. A transport computed from the endpoints' own channels has rank at most the number of channels, never a group element from one vector each, so relations must be supplied; the standard transformer is the flat sector of one shared frame. Representation theory then attaches an integer to each axis of a mechanism, and the readout decides which integer. A readout with convex decision cells composes a finite group exactly only at widths of at least its smallest faithful real representation, and at that width every exact solution is a faithful representation with a finite latent group---as a conclusion, not a hypothesis. An even readout can decode a cover of the group instead, at a smaller width, and the gap is unbounded. Below its integer a mechanism is capped by the part of the group it cannot see---the commutator subgroup for a commutative mechanism, the derived series for a stack of commutative layers, the lower central series for a signature readout---all computable from the group before training. Empirically, every fitting seed at the minimal width generalizes exactly to length $1000$ ($17$ of $17$, against $7$ of $18$ at twice the width), with certified horizons read from the weights. On $C_3\times A_4$, a $36$-element group whose faithful width is $5$, gradient descent at width $4$ learned the $72$-state double cover under a quadratic readout and under sign-blind attention ($6$ of $48$ seeds, half declared in advance; monotone attention $0$ of $48$), and on a frozen cover the cross-entropy-optimal affine head is provably uniform while a quadratic head is exact. On parsed CLUTRR, transported composition beats a matched transformer by $41$ points over four to ten hops (trained on two and three), and forcing isometries reverses that by $65$---no isometry represents an idempotent rule such as sibling of sibling, as Krohn--Rhodes predicts---while on StepGame's commutative algebra isometries help instead. Static embeddings gave each word one vector; contextual models gave each occurrence its own; per-entity frames make each relation an operator, the unit in which relations compose. This points toward attention that produces its operators from context rather than retrieving them from a library, and the carrier theorem writes its specification in advance: not from the two entities, but from a carrier the context supplies.
\end{abstract}

\section{Introduction}
Attention compares vectors with a dot product, and a dot product is meaningful only when the two vectors are written in the same coordinate system. That assumption is natural for embeddings learned in one vocabulary, but less natural for a newly introduced name, object, or program variable whose role is specified through relations. We study an explicit alternative: changing the coordinates of one entity, while transforming its attached relation data consistently, should not change the prediction. This is an inductive bias, not a claim that fresh names logically imply a physical symmetry.

The premise leads to a simple rule: \emph{transport, then compare}. Each entity owns its frame, so a relation operator must carry a neighbour's vector into the receiver's frame before it enters a score or message (Figure~\ref{fig:concept}), a translator between two coordinate languages, of which standard attention is the case where every entity already speaks the same one. The form itself has fifteen years of empirical precedent, in knowledge-graph embedding, orthogonal recurrences and local-frame attention~\citep{nickel2011,sun2019,he2021,miyato2024}; what is ours is the derivation from the premise and what follows from it. Our questions are what the premise requires, what information can supply the operator, what exact composition requires of the resulting memory, and what a mechanism short of that requirement can still do.

\begin{figure}[t]
\centering
\includegraphics[width=0.96\linewidth]{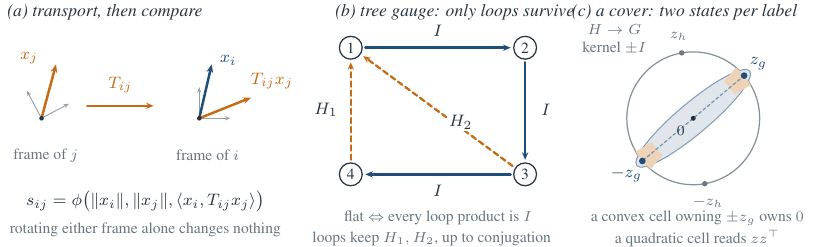}
\caption{\textbf{Frames, loops and covers.} (a) Under per-entity frames, with the relation data transformed consistently, an invariant score can depend on the pair only through the norms and the transported inner product $\ip{x_i}{T_{ij}x_j}$: transport, then compare (Theorem~\ref{thm:forced}). (b) A spanning-tree gauge makes every tree transport $I$; only the loop products $H_1,H_2$ survive, up to one common conjugation, and the connection is flat iff they are all $I$ (Section~\ref{sec:flat}). (c) An exact memory need not represent $G$ itself: a finite $H\to G$ with kernel $\pm I$, acting freely on a generic initial state, has two states $\pm z_g$ per label; a convex decision cell owning both owns their midpoint $0$, so no convex-cell readout decodes the cover, while a quadratic readout, which reads $zz^{\top}$ and so cannot tell $z$ from $-z$, decodes it (Theorem~\ref{thm:graded}).}
\label{fig:concept}
\end{figure}

The last two questions have one answer: the pair (width, readout) selects the algebra the model can realize. With convex decision cells, exact tracking of a finite group needs its smallest faithful real representation, and at that width exactness forces the learned transitions to be the group itself, their latent closure finite as a conclusion. Below the width a transport does not degrade gracefully: its accuracy is capped by the largest quotient the width admits, the kernel it cannot see, and in our runs descent lands on that quotient; or, with a readout blind to sign, it can realize a cover of the group, more states in fewer dimensions. For $C_3\times A_4$ the faithful width is five, but a double cover works in four; for $A_4\times A_4$, six and four; and gradient descent at width four finds the cover, certified, when the readout can read it and not when it cannot.

Sections~\ref{sec:frames}--\ref{sec:cert} prove this chain: the forced form of scores and messages and the rank of a transport computed from endpoint channels (Section~\ref{sec:frames}); convex-readout rigidity, the readout-graded width, and the ceilings below the integer, which classify tasks by their algebra before training (Section~\ref{sec:width}); and finite certificates that separate a verified snapped model from the learned weights and bound the latter to an explicit horizon (Section~\ref{sec:cert}). Section~\ref{sec:exp} measures all of it, the frozen-cover intervention most directly: the states are identical, and changing only the head changes what can be decoded.

Throughout, a proof of existence is not a guarantee that descent finds the construction, and finite-test accuracy is not all-word exactness: unsuccessful runs are reported, exploratory protocols kept separate, and the raw weights given worst-case horizons.

\section{Independent frames and transported attention}\label{sec:frames}\label{sec:premise}\label{sec:theory}
Entity $i$ carries a vector $x_i\in V_i\cong\R^d$. Independent frames act by $x_i\mapsto g_ix_i$, where $g_i\in\Od$. An invariant scalar is unchanged; an equivariant output at $i$ transforms by $g_i$. A transport $T_{ij}:V_j\to V_i$ transforms as
\begin{equation}\label{eq:gauge}
T_{ij}\mapsto g_iT_{ij}g_j^{-1},\qquad T_{ij}x_j\mapsto g_iT_{ij}x_j.
\end{equation}
The metric makes $\Od$ the frame group; invariant attributes and labels may also be supplied.

\begin{samepage}
\subsection{What the symmetry forces}
\begin{theorem}[Scores and messages]\label{thm:forced}
For independent endpoint frames and $T_{ij}\in\Od$, a continuous score is invariant if and only if
\[
s_{ij}=\phi\bigl(\norm{x_i},\norm{x_j},\ip{x_i}{T_{ij}x_j}\bigr)\quad\text{for some continuous }\phi.
\]
A continuous message from $j$ to $i$ is equivariant if and only if it has the form $m_{ij}=\alpha x_i+\beta T_{ij}x_j$, with invariant scalar coefficients $\alpha,\beta$ (continuous on constant-rank strata of the two-vector span).
\end{theorem}
\end{samepage}
\begin{proof}[Proof sketch]
The independent-frame orbits of $(x_i,x_j,T_{ij})$ correspond to the common-frame orbits of $(x_i,T_{ij}x_j)$. The latter are separated by their Gram matrix. The stabilizer of the two vectors fixes no direction outside their span. These are standard orthogonal invariant-theory facts~\citep{weyl1946,villar2021}; Theorem~\ref{note:thm:forced} gives the full argument.
\end{proof}
For normalized vectors the score is a function of the transported inner product alone; the theorem specifies admissible arguments, not a nonlinearity or an architecture. For example, with attention weights $a_{ij}$ and output $y_i$,
\begin{equation}\label{eq:attention}
a_{ij}=\operatorname{softmax}_j(s_{ij}),\qquad
y_i=\sum_j a_{ij}T_{ij}x_j
\end{equation}
is equivariant. Score and value share the transport here because one connection is supplied; symmetry permits distinct covariant carriers. For nonorthogonal operators, further invariants and covariants are needed (Proposition~\ref{note:prop:gl}). The theorem says what the transport must do, not what it must be: an isometry, or a non-invertible map, which we call a \emph{reset}; which kind a task needs is decided by the algebra of its relations, and both kinds are required on public data (Section~\ref{sec:public}). The empirical face of the theorem is one cell: under independent random rotations of the entities' inputs, a fixed-frame bilinear score reaches $0.19$ on a relational task where chance is $0.17$, an MLP score $0.16$, and the forced score $1.000$ (three seeds; Appendix~\ref{app:m5}).

\subsection{Where a transport can come from}
\begin{theorem}[Endpoint carriers]\label{thm:carrier}
Let the endpoints carry channels $X_i\in\R^{d\times m_i}$ and $X_j\in\R^{d\times m_j}$, and let $e$ be invariant data. Every continuous covariant map computed only from these inputs has the form
\[
T(X_i,X_j,e)=X_iC(X_i^\top X_i,X_j^\top X_j,e)X_j^\top,
\qquad \rank T\le\min(\rank X_i,\rank X_j).
\]
The coefficients may be chosen invariant and continuous on fixed-rank strata. In particular, one vector per endpoint supplies at most a rank-one carrier, never an orthogonal transport when $d\ge2$.
\end{theorem}
\begin{proof}[Proof sketch]
The stabilizer of $X_i$ rotates its orthogonal complement freely, so every column of $T$ must lie in $\operatorname{col}(X_i)$. Applying the same argument at $j$ gives $T=P_iTP_j$. Substitution of $P_i=X_iX_i^+$ gives the factorization (Theorem~\ref{note:thm:carrier}).
\end{proof}
With square invertible channel matrices the polar frames $F_i=X_i(X_i^\top X_i)^{-1/2}$ give $T_{ij}=F_i\widetilde C F_j^\top$ (rectangular full-rank inputs admit such frames on charts, not globally), so a relation-label library supplies operators relative to a frame choice; our experiments use a common fixed frame. The theorem concerns endpoint-only inputs: supplied graph transports bring further anchors to an endpoint, and one generic anchor with a cycle of simple spectrum generates a full frame (Lemma~\ref{note:lem:framesynth}).

\subsection{Shared frames and graph composition}\label{sec:flat}
For an inverse-consistent orthogonal connection, all cycle products equal the identity if and only if $T_{ij}=F_iF_j^{-1}$ for some $F_i\in\Od$ per entity, equivalently there is a gauge with every $T_{ij}=I$. We call this the \emph{flat sector}. Standard attention operates in a shared frame: a head with query, key, value and output matrices $W_Q,W_K,W_V,W_O$ has $W_Q^\top W_K$ and $W_OW_V$ as tensors on that common space, rather than independently transforming edge transports. This identifies the geometry; it does not restrict what a deep transformer can compute.

On a graph, composition along a path is a product of transports; a spanning-tree gauge leaves only the fundamental loop products, or holonomies, up to simultaneous conjugation, and traces of words in them separate connection orbits~\citep{procesi1976,gao2021} (Proposition~\ref{note:prop:hol}).

\section{The readout determines the composition width}\label{sec:width}
Let letters $a\in\Sigma$ generate a finite target group $G$ through $g_a$. A transport recurrence assigns orthogonal matrices $R_a$ and an initial state $x_0$, with
\begin{equation}\label{eq:model}
x_w=R_{a_L}\cdots R_{a_1}x_0,\qquad
g_w=g_{a_L}\cdots g_{a_1},\qquad w=a_1\cdots a_L.
\end{equation}
A decoder $r$ is \emph{exact} when $r(x_w)=g_w$ for every finite word, including the empty word. A \emph{convex-cell} decoder has pairwise disjoint convex sets $C_g$ covering the reachable states, with $r=g$ on $C_g$. Affine argmax, one-prototype-per-class nearest-neighbor decoding and argmax over softmax probabilities of affine logits satisfy this condition, whereas general attention followed by arbitrary value aggregation need not have convex cells.

\subsection{Exactness forces a faithful representation}
Let $\dmin(G)$ denote the smallest dimension in which $G$ has a faithful real representation, the least width at which orthogonal transitions can realize $G$ at all.
\begin{theorem}[Convex-readout width and rigidity]\label{thm:width}
Suppose the orthogonal recurrence~\eqref{eq:model} has an exact convex-cell decoder. Let $H=\overline{\langle R_a\rangle}$ be the closure of the group generated by the transitions. There is a continuous surjective homomorphism $\varphi:H\to G$ with $\varphi(R_a)=g_a$, and $G$ acts faithfully on $\Fix(\ker\varphi)=\{v:nv=v\text{ for all }n\in\ker\varphi\}$, the subspace the kernel fixes, so
\[
d\ge\dim\Fix(\ker\varphi)\ge\dmin(G).
\]
At $d=\dmin(G)$, $\varphi$ is an isomorphism: the latent group is finite and the transitions form a faithful representation. Conversely, a faithful representation with a generic initial state admits exact affine-argmax decoding.
\end{theorem}
\begin{proof}[Proof sketch]
Let $K$ be the closure in $G\times\Od$ of the pairs $(g_w,R_w)$ over all words; it is a compact group, and its slice over $g\in G$ is the set of latent transitions that can carry the label $g$. Suppose the slice over some $g\ne e$ contained the identity, $(g,I)\in K$. Then the states reachable with label $e$ and those reachable with label $g$ have the same closure, hence, because Haar measure on a compact group has full support, the same mean, and each label's convex cell contains the mean of its own reachable states; two disjoint convex cells cannot share a point, so no such $g$ exists. The slice over $e$ is therefore the kernel of a homomorphism $\varphi:H\to G$. Averaging over that kernel is the orthogonal projection $P$ onto its fixed subspace; the projected states $Phx_0$ keep their labels, so $G$ acts freely on them and faithfully on $\Fix(\ker\varphi)$, giving $d\ge\dim\Fix(\ker\varphi)\ge\dmin(G)$. At $d=\dmin$ the projection is the identity, the kernel is trivial, and the transitions are a faithful representation. Theorem~\ref{note:thm:widthcompact} supplies the details, including nonclosed cells.
\end{proof}
No finite-order, margin, or finite-latent-group assumption is needed. For $SO(d)$ transitions the minimum is the least faithful special-orthogonal dimension $s=\dmin^{SO}(G)\in\{\dmin,\dmin+1\}$, and exactness at $d=s$ again forces a faithful representation (Corollary~\ref{note:cor:so}); for $D_4$ the two minima are two and three.

The theorem concerns all words: a rotation displaced slightly from a root of unity decodes correctly for thousands of steps before drifting across a decision boundary, so fitting a finite sample does not meet its hypothesis, and verification rather than finite accuracy connects a learned model to it (Section~\ref{sec:cert}).

\subsection{Nonlinear readouts can decode a smaller cover}
\begin{table}[t]
\centering\small
\resizebox{\textwidth}{!}{\begin{tabular}{@{}lcccccccc@{}}
\toprule
group & $|G|$ & commutative ceiling $1/|[G,G]|$ & derived length & nilpotency class & $\dmin^{O}$ & $\dmin^{SO}$ & $\dproj$ & $\dcov$\\
\midrule
$D_4$ & 8 & $0.5$ & $2$ & $2$ & $2$ & $3$ & $2$ & $2$\\
$D_{50}$ & 100 & $0.04$ & $2$ & --- (cap $0.04$) & $2$ & $3$ & $2$ & $2$\\
$A_4$ & 12 & $0.25$ & $2$ & --- (cap $0.25$) & $3$ & $3$ & $3$ & $3$\\
$S_4$ & 24 & $0.083$ & $3$ & --- (cap $0.083$) & $3$ & $3$ & $3$ & $3$\\
$A_5$ & 60 & $0.0167$ & never (perfect) & --- (cap $0.0167$) & $3$ & $3$ & $3$ & $3$\\
$Q_8$ & 8 & $0.5$ & $2$ & $2$ & $4$ & $4$ & $4$ & $4$\\
$S_5$ & 120 & $0.0167$ & never & --- (cap $0.0167$) & $4$ & $5$ & $4$ & $4$\\
$C_3\times A_4$ & 36 & $0.25$ & $2$ & --- & $5$ & $5$ & $4$ & $4$\\
$A_5\times A_5$ & 3600 & $1/3600$ & never & --- & $6$ & $6$ & $4$ & $4$\\
\midrule
kinship rules & --- & --- & --- & --- & \multicolumn{4}{c}{resets required (idempotent rules, e.g.\ sibling of sibling)}\\
\bottomrule
\end{tabular}}
\caption{\textbf{Selection rules}, computed from each group's Cayley table before training: the limiting ceiling of any commutative or order-blind mechanism; the depth at which a commutative cascade becomes exact; the signature level at which a signature readout becomes exact (a dash: not nilpotent, every level capped); the minimal faithful width over $\Od$ and $SO(d)$ (convex-cell readouts); $\dproj$, the width at which a quadratic or sign-blind readout can decode a double cover ($C_{2n}/\pm I\cong C_n$ and $D_{2n}/\pm I\cong D_n$ give $2$ for cyclic and dihedral groups; an order-$16$ subgroup of $O(4)$ gives $4$ for $Q_8$; for $S_5$ the standard representation has trivial centre and $PO(3)\cong SO(3)$ has no $S_5$); and $\dcov$ (arbitrary readouts). The three widths first separate at $\dmin=5$. The last row is CLUTRR's kinship composition rules (Section~\ref{sec:public}), which contain idempotents and so need non-invertible transports.}
\label{tab:selection}
\end{table}
Define $\dcov(G)$ as the least $d$ for which a finite $H\le\Od$ surjects onto $G$, and $\dproj(G)$ as the least $d$ for which $G$ embeds in $PO(d)=\Od/\{\pm I\}$. For decoder class $\mathcal R$, let $w_{\mathcal R}$ be the least exact width among models with a finite latent group.
\begin{theorem}[Readout-dependent width]\label{thm:graded}
For affine argmax, degree-at-most-two polynomial argmax, and arbitrary decoders (widths $w_{\rm lin}$, $w_{\rm quad}$, $w_{\rm any}$),
\[
w_{\rm lin}=\dmin,\qquad
w_{\rm quad}\le\min(\dmin,\dproj),\qquad
w_{\rm any}=\dcov.
\]
For $C_3\times A_4$, the cyclic group of order three times the tetrahedral rotation group ($36$ elements), $w_{\rm any}=w_{\rm quad}=4<5=w_{\rm lin}$; for $A_4\times A_4$ and $A_5\times A_5$, the corresponding separation is $4<6$. On $A_5^{2m}$, linear width is $6m$ while quadratic width is at most $4m$, giving an unbounded additive gap.
\end{theorem}
\begin{proof}[Proof sketch]
The upper bounds are constructions. Lift a projective embedding $G\hookrightarrow PO(d)$ to a finite subgroup $H\le\Od$ with kernel $\{\pm I\}$; for a generic initial state the feature $x_wx_w^\top$ identifies the target element while discarding the sign of the lift, and distinct feature prototypes are affinely separable, so a quadratic readout is exact at width $\dproj$. The identity $PSO(4)\cong SO(3)\times SO(3)$ supplies the four-dimensional examples, and independent blocks give the family $A_5^{2m}$. For the lower bounds, an arbitrary readout with a finite latent group $H$ still needs $H\twoheadrightarrow G$, so its width is at least $\dcov$, and an affine readout is a convex-cell readout, so Theorem~\ref{thm:width} gives $\dmin$ (Theorem~\ref{note:thm:graded}).
\end{proof}
The saving is in state dimension, not computation (a quadratic head at width four has ten features). A squared-prototype head, logits $\beta\ip{p_g}{x}^2$ with one prototype $p_g$ per class and scale $\beta$, is an attention-form readout insensitive to $x\mapsto-x$; we compare it with a monotone prototype score on the same recurrence.

\paragraph{A concrete double cover.}
For $G=C_3\times A_4$ the tetrahedral group $A_4$ has the binary tetrahedral double cover $2T$; $H=C_3\times2T$ has $72$ elements acting faithfully and orthogonally on $\R^4$, and quotienting by its central $\{\pm I\}$ gives the $36$ labels, so a generic initial state has two lifts $\pm z_g$ per label: more states in fewer dimensions. A target involution $b$ lifts to an update with $R_b^2=-I$: two repetitions return the label and reverse the state, which quadratic features ignore since $(-z_g)(-z_g)^{\top}=z_gz_g^{\top}$. This also shows the convex obstruction without dimension counting (Figure~\ref{fig:concept}c): a convex cell containing both lifts contains their midpoint $0$, and two labels cannot both own $0$. The cover is an exact compositional memory that no convex-cell head decodes and a squared-prototype score does; the experiment below tests exactly this on one frozen set of states.

The finite-latent assumption in Theorem~\ref{thm:graded} matters for arbitrary readouts: without it an infinite free subgroup and a discontinuous lookup encode the word itself at low dimension, so the result is not a memory lower bound for unrestricted computation.

\paragraph{A decoder intervention with a sharp prediction.}
On a balanced antipodal cover with $M$ distinct label pairs $\{\pm z_g\}$, an affine-softmax head with weights $W$ and biases $(b_h)$ has cross-entropy
\begin{equation}\label{eq:frozen}
\mathcal L(W,b)\ge\log\sum_h e^{b_h}-\frac1M\sum_gb_g\ge\log M.
\end{equation}
The first inequality averages the logits of the two lifts and applies Jensen's inequality to the convex log-sum-exp; the second uses balanced labels. Equality predicts uniformly on the orbit. Quadratic features distinguish the label pairs. This concerns cross-entropy, not maximal affine classification accuracy: a nearest-selected-lift classifier can score $1/2$ despite a worse loss. The proof and kernel-averaging generalization are in Proposition~\ref{note:prop:frozen}.

\subsection{Below the integer: ceilings}\label{sec:ceilings}
\begin{theorem}[Ceilings, joint-process form]\label{thm:ceiling}
Let the letters act on $G\times H$ by $(g_a,R_a)$ with $H$ finite, let $K$ be the joint group they generate and $N_G=\pi_G(K\cap(G\times\{I\}))$ the part of $G$ the latent process cannot see. On uniform random words the accuracy of any readout of the latent state is at most $1/|N_G|$ plus a term that vanishes as the joint walk mixes (so the ceilings below are limits under mixing, not finite-length bounds), and the bound is attained by realizing the quotient. Consequently:
\begin{list}{}{\setlength{\leftmargin}{1.6em}\setlength{\labelwidth}{1.3em}\setlength{\labelsep}{0.3em}\setlength{\itemsep}{0pt}\setlength{\parsep}{0pt}\setlength{\topsep}{1pt}\setlength{\partopsep}{0pt}}
\item[(a)] any mechanism whose transitions commute is capped at $1/|[G,G]|$, $[G,G]$ the commutator subgroup, and an order-blind function of the word is bounded at every length by the exact multiset-optimal accuracy $c_G(L)$, which tends to the same limit;
\item[(b)] $k$ cascaded commutative layers are capped at $1/|G^{(k)}|$, $G^{(k)}$ the $k$-th derived subgroup, reaching $1$ at the derived length;
\item[(c)] a level-$k$ signature decodes exactly only for nilpotency class at most $k$, and a readout of the level-$k$ nilpotent quotient is capped at $1/|\gamma_{k+1}G|$, $\gamma_{k}G$ the lower central series;
\item[(d)] a transport of width $d<\dcov(G)$ is capped by $\max|Q|/|G|$ over quotients $Q$ with $\dcov(Q)\le d$ (Appendix~\ref{app:proofs}, \ref*{note:thm:ceiling}--\ref*{note:prop:sig}).
\end{list}
\end{theorem}
\begin{proof}[Proof sketch]
By Goursat's lemma $K$ projects onto both factors with normal kernels, and the latent state after a word depends only on the coset $g_wN_G$ of its label; so a readout of that state is right on at most one label per coset, giving $1/|N_G|$ once the walk has mixed over each coset. Commuting transitions generate an abelian $H$, so $[G,G]\subseteq N_G$, and a multiset function is a commutative mechanism; the derived and lower central series are the abelian-tower and nilpotent-tower kernels (Chen's theorem for signatures). The order-blind bound is concurrently proved by \citet{zhangli2026}; (b) sharpens the threshold of \citet{shakerinava2026} into the ceiling below it.
\end{proof}
The depth and signature series part at $A_4$, solvable but not nilpotent: two cascaded commutative layers compose it exactly, while no signature level decodes it and, under the hypotheses of Theorem~\ref{thm:ceiling}(c), every level is capped at $0.25$; the constructed cascade is exact at depth $2$ and learned signature readouts sit at chance under the cap (Appendix~\ref{app:additional}). Nothing here depends on how a task is presented, only on the group its relations generate and on the readout and mechanism class: Table~\ref{tab:selection} lists the integers for the groups we measure, computed before training, and the Krohn--Rhodes distinction~\citep{krohn1965} between groups and monoids with resets is its last row, for the composition rules of CLUTRR rather than a monoid identified from data.

\section{From exact constructions to learned-model guarantees}\label{sec:cert}
Learned matrices have nonzero residuals, so we distinguish an algebraically verified \emph{snapped model} from the original recurrence and its finite evaluation.
\begin{theorem}[Finite prototype certificate]\label{thm:cert}
Let a finite set $Z$ contain $z_0=x_0$, with transitions $\delta:Z\times\Sigma\to Z$ and labels $\lambda:Z\to G$ satisfying $\lambda(z_0)=e$ and $\lambda(\delta(z,a))=g_a\lambda(z)$. Suppose
\[
\epsilon=\max_{z,a}\norm{R_az-\delta(z,a)},\qquad
r\equiv\lambda(z)\text{ on the ball }B(z,\eta)\quad\text{for all }z\in Z.
\]
Then every word of length $L$ with $L\epsilon<\eta$ is decoded correctly, and every word if $\epsilon=0<\eta$.
\end{theorem}
\begin{proof}
Run the prototype automaton alongside the recurrence on the same word. Because the transitions are orthogonal, the distance from the learned state to its prototype grows by at most $\epsilon$ per letter, so $\norm{x_w-z_w}\le|w|\epsilon<\eta$; the readout therefore returns $\lambda(z_w)=g_w$ by label consistency.
\end{proof}
The certificate checks $|\Sigma||Z|$ transitions and $|Z|$ readout neighborhoods; for linear logits $\eta$ is the minimum margin over the norm of the weight difference, and quadratic scores admit a norm bound or a tighter trust-region radius (Theorem~\ref{note:thm:proto}). Orbit-indexed prototypes handle several lifts per label; proposing a snapped representation is separate from verifying it.

The guarantees concern real arithmetic with the stored matrices (if $\norm{R_a}\le1+\kappa$, replace $L\epsilon$ by $\epsilon\sum_{t<L}(1+\kappa)^t$); floating-point execution would need arithmetic-error bounds besides.

\section{Experiments}\label{sec:exp}
We separate three questions: does the model fit, does it generalize on sampled words, and what can be certified? Unless stated otherwise, \emph{fit} means accuracy at least $0.99$ at length ten; \emph{evaluation success} means $1.000$ on 2,000 fresh words at each reported length through 1,000. Neither means all-word exactness. Protocols and full cell counts appear in Appendix~\ref{app:exp}.

\subsection{Width and verification}\label{sec:width_exp}
The eight-group ladder runs from cyclic and dihedral groups to $Q_8$, $A_5$ and $S_5$ (Appendix~\ref{app:ladder}), with predicted orthogonal widths from two to four; order alone does not determine width, since $Q_8$ needs four dimensions and the larger $A_5$ three. Every group fits at its predicted width, and raw generalization varies (recorded $L=300$ accuracies $0.549$ for $D_{50}$, $0.746$ for $Q_8$, $0.843$ for $S_5$): the ladder supports representability and extraction, not uniformly reliable optimization.

For a more controlled comparison, $A_4$ and $D_4$ were trained with $SO(d)$ transitions at their minimum special-orthogonal width $d=3$ and at $d=6$, using anchor, fit, and polish stages. Among fitting runs in the reported grid, $17/17$ at $d=3$ passed the finite evaluation criterion, compared with $7/18$ at $d=6$. A separate certificate cohort of six seeds per group at $d=3$ is exact to length $1000$ in every seed, and the snapped generators of every seed close to a group of $|G|$ elements (their products form exactly $|G|$ matrices with a Latin-square multiplication table) with homomorphism defect at most $6\times10^{-13}$, exact at $10^4$; with the recorded radius the raw weights are certified to $71$--$276$ letters ($A_4$) and $100$--$222$ ($D_4$), and the first observed error on a $200{,}000$-letter walk comes at $7{,}253$--$73{,}076$ letters for three $A_4$ seeds and never for three, and at $10{,}541$--$67{,}191$ for $D_4$ (Appendix~\ref{app:ladder} and Figure~\ref{fig:minwidth}; every certified run lies above the diagonal of Figure~\ref{fig:horizons}).

The gap between worst-case and observed horizons is expected, since the certificate charges the largest residual at every step while realized residuals are smaller and partly cancel; conversely a $Q_8$ run that passed the length-$1{,}000$ test failed at $2{,}000$, so long finite success establishes no exact representation.

\subsection{A learned cover below faithful width}\label{sec:spinor}
The decisive group is $C_3\times A_4$ (convex decoder: width five; even decoder: a four-dimensional double cover); the registered protocol tested twelve seeds per cell. Two width-four quadratic runs passed the evaluation criterion and were identified with a 72-state cover over 36 labels. For the target involution $b$ the raw matrices satisfy $\norm{R_b^2+I}_F\le3.6\times10^{-3}$, and $10^{-15}$ after snapping: the cover with kernel $\{\pm I\}$ and $R_b^2=-I$. The raw models drift at longer lengths: certified horizon $15$, first observed errors at $1{,}737$ and $2{,}872$ letters (minimum over five walks).

\begin{figure}[t]
\centering
\begin{minipage}[t]{0.5\textwidth}
\vspace{0pt}\centering\scriptsize\setlength{\tabcolsep}{3pt}
\begin{tabular}{@{}llrr@{}}
\toprule
Protocol & Readout and width & Fit & Success\\
\midrule
Registered & quadratic, $d=4$ & 2/12 & 2/12\\
Registered & affine, $d=4$ & 0/12 & 0/12\\
Registered & prototype, monotone, $d=4$ & 0/12 & 0/12\\
Registered & prototype, even, $d=4$ & 0/12 & 0/12\\
Registered & affine, $d=5$ & 6/12 & 0/12\\
\midrule
Exploratory & cosine, even, $d=4$ & 3/24 & 3/24\\
Exploratory & cosine, monotone, $d=4$ & 0/24 & 0/24\\
Exploratory & cosine, monotone, $d=5$ & 7/12 & 6/12\\
\midrule
Confirmatory & cosine, even, $d=4$ & 3/24 & 3/24\\
Confirmatory & cosine, monotone, $d=4$ & 0/24 & 0/24\\
Confirmatory & cosine, monotone, $d=5$ & 7/24 & 7/24\\
\bottomrule
\end{tabular}
\end{minipage}\hfill
\begin{minipage}[t]{0.47\textwidth}
\vspace{0pt}\centering
\includegraphics[width=\linewidth]{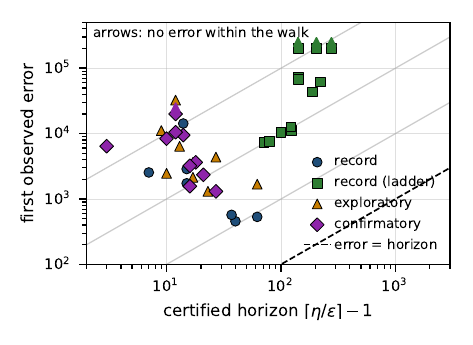}
\end{minipage}
\caption{\textbf{Left: readout parity and optimization} on $C_3\times A_4$; fit is $\ge0.99$ at length ten, evaluation success is $1.000$ on $2{,}000$ words at every length through $1{,}000$. The registered prototype-attention cells failed to fit; exploratory cells use cosine scores, initial scale ten and doubled training; confirmatory cells re-run them with byte-identical code on seeds $24$--$47$ declared in advance. Every successful width-four even cell is a certified double cover (Appendix~\ref{app:spinor}). \textbf{Right: observed errors occur beyond the computed horizons.} Certified horizon of the raw weights (Theorem~\ref{thm:cert}, record radius) against the first observed error on a random walk, one point per certified run of Table~\ref{tab:evidence} with a positive horizon ($39$ runs); every point lies above the diagonal, by a factor of $9$ to $2{,}711$ (median $201$) over the $35$ uncensored runs; arrows mark walks that found no error.}
\label{fig:horizons}
\end{figure}

The registered attention-form heads did not fit at either parity. An exploratory protocol then used cosine logits $\beta\cos(p_g,x)$ and $\beta\cos^2(p_g,x)$, $\cos(p_g,x)$ the cosine of the angle between prototype and state, initial scale $\beta=10$ and twice the training budget: the even head produced three width-four covers in $24$ attempts, the monotone head none there and six of twelve at width five. This is consistent with the expressivity separation, with the small basin and post hoc protocol limiting the trainability claim; a zero-success finite experiment is not itself required by the all-word theorem.

The three cosine cells were then re-run on seeds $24$--$47$, declared in advance, with byte-identical code (hashes in a declaration written before the first run) against a pass criterion written down beforehand (Appendix~\ref{app:spinor}). No run was excluded. Sign-blind attention at width four: $3$ of $24$ exact, the sweep's rate, each a certified double cover (margins $12.1$--$12.7$, certified horizons $10$--$14$, first raw errors at $8{,}362$--$10{,}484$); monotone at width four $0$ of $24$, none fitting; monotone at width five $7$ of $24$, all faithful. Pooled: sign-blind $6$ of $48$, monotone $0$ of $48$ at width four and $13$ of $36$ at width five. The replication makes the result confirmatory for this variant without making the variant pre-registered: it independently reproduced cover learning at the sweep's rate on declared seeds, and every exact seed is the certified cover.

Two controls qualify the result: at width five the two fitting quadratic runs selected a padded width-four cover although a faithful representation fits, and on $A_4\times A_4$ the width-four quadratic cell had $0/12$ fitting runs, so that $6$-to-$4$ separation remains a construction (width-six linear control: $2/12$).

\paragraph{Frozen states isolate the decoder.}
For two registered covers and one exploratory cover, we enumerate 72 states and fit only the head with L-BFGS and five restarts. On the snapped balanced orbit, affine cross-entropy is $3.5835=\log36$ to four decimals, as predicted by~\eqref{eq:frozen}. Full quadratic and squared-prototype heads attain accuracy $1.000$ and cross-entropy at most $10^{-5}$ on the same states. A nearest-selected-lift affine head attains $0.500$ accuracy, the cross-entropy optimum differing from the accuracy optimum; the raw-orbit intervention, whose approximate symmetry is outside the theorem's hypotheses, is in Appendix~\ref{app:frozen}.

\subsection{Transport constraints on parsed relational tasks}\label{sec:public}
StepGame~\citep{shi2022} and CLUTRR~\citep{sinha2019} are evaluated through common parsed interfaces: every model receives the same extracted relations. These experiments concern composition given relations, not text understanding or text-level state of the art; accuracy is the mean out of distribution (OOD), on more hops than were trained (StepGame: train $k\le5$, test $6$--$10$; CLUTRR: train $k\in\{2,3\}$, test $4$--$10$; every model and control in Table~\ref{tab:public_full}).

StepGame's displacements compose in an abelian algebra. An additive composer with a difference readout scores $0.753$ against the transformer's $0.481$, and within the source-initialized probe family replacing general linear maps by rotations raises the mean from $0.73$ to $0.82$ (the composer and probe rows use different readouts and are not one comparison); commutativity alone does not make every attention weighting a mean aggregator.

CLUTRR gives the opposite intervention: unconstrained matrix scans score $0.848$ OOD against the transformer's $0.436$, and $0.202$ when constrained to be orthogonal. The motivation is algebraic: a nonidentity idempotent has no invertible representation, so rules with resets need non-invertible transports, the Krohn--Rhodes distinction~\citep{krohn1965} of Table~\ref{tab:selection}. It does not identify CLUTRR's relation monoid; it shows that one isometry constraint helps one parsed task and harms the other, so transport structure must be tested against the task.

\section{Related work and scope}\label{sec:related}
\paragraph{Local frames and relational operators.}
GET~\citep{he2021}, Neural Sheaf Diffusion~\citep{bodnar2022}, Bundle Neural Networks~\citep{bamberger2025} and geometry-aware attention~\citep{miyato2024} use local spaces and transported features, BiEquiFormer~\citep{pertigkiozoglou2024} aligns keys and values across rotated inputs; bilinear scores and path products predate these in knowledge-graph models~\citep{nickel2011,guu2015,sun2019,xu2019}. Connection Laplacians and synchronization~\citep{singer2012,gao2021} and the operator view of attention of \citet{lin2026} are the geometric background.

\paragraph{Composition, width, and readout.}
Orthogonal recurrences~\citep{arjovsky2016,wisdom2016} and state-tracking models built from reflection products or structured sparse transitions~\citep{grazzi2025,siems2025,terzic2025} are the primitives. \citet{larson2026}, concurrently, trains a matrix memory toward a prescribed $\dmin$-dimensional representation under a fixed readout and finds the recruited rank equal to $\dmin$ on five groups, necessity there being the rank of the prescribed target, scored at training lengths. Theorems~\ref{thm:width}--\ref{thm:graded} prescribe no target and fix no readout: exactness under a convex-cell readout forces a faithful representation, and a quadratic readout escapes through a cover, which a fixed decoder cannot show. \citet{howe2026} and \citet{nebli2026} note readout dependence and a quadratic-versus-affine separation in other settings; the finite-group cover and the compact-closure necessity are ours. \citet{nowak2026} give the dynamics/readout view underlying group division. The cascade threshold of \citet{shakerinava2026} is the exactness case of Theorem~\ref{thm:ceiling}(b). Transformers compute noncommutative composition by other mechanisms~\citep{liu2023}; nothing here says they are order blind.

\section{Discussion and limitations}\label{sec:outlook}
One premise, that each entity owns its frame, forces the arc of this paper. Frame ownership admits directional information only through covariant transport, whose orthogonality constrains the latent algebra to a quotient, a cover or a faithful copy of the target; which of these a model realizes depends on the pair (width, readout) and, in our runs, on which of them gradient descent found. The certificate, where applied, separates exact structure from a nearby approximation. Each step rests on explicit hypotheses about one model class.

The sharpest consequence is the readout's part in it. On the same $72$ frozen states the cross-entropy-optimal affine head is provably uniform and the quadratic head exact; in training, the cover was found only under readouts that can use it ($6$ of $48$ sign-blind, $0$ of $48$ monotone). The representation is a property of dynamics and decoding together: what a model builds depends on what its decoder can read. Whether this holds beyond transport models is the question the paper leaves open.

The limitations of this work mark the boundary of what is currently established. The groups are small because finiteness makes the theorems exact, the certificates finite and the integers computable before training; continuous or large groups need new tools. Optimization success is rare and its determinants are not predicted by the theory; the attention-form variant was chosen exploratorily and then tested on prospectively declared seeds. The benchmarks use parsed relations, and their algebraic reading is tentative. The distance to a transformer stack is explicit (Proposition~\ref{note:prop:flat}) and not crossed here.

The carrier theorem also poses the next step. With exact matched anchors $B=QA$ an orthogonal transport is identified only on the anchored subspace, since $BA^{+}=QAA^{+}=QP_{\operatorname{col}(A)}$; inferring transports from incomplete relational evidence and propagating the remaining uncertainty through attention is the problem that follows, and the present experiments, with supplied relation operators, do not evaluate it. Within this class, behavior is governed by integers computed before training; how far that governance extends beyond the class is the open question.

\label{maintext:end}

\clearpage
\appendix
\section{Full statements and proofs}\label{app:proofs}
\setcounter{theorem}{0}
\renewcommand{\thetheorem}{A.\arabic{theorem}}
\renewcommand{\theHtheorem}{A.\arabic{theorem}}
The full statements below include the supporting results omitted from the main text; every statement carries its hypotheses. Main-text results and their full versions:
\begin{center}\footnotesize
\begin{tabular}{@{}p{4.2cm}p{9.2cm}@{}}
\toprule
main text & full statement and proof\\
\midrule
Theorem~\ref{thm:forced} (scores and messages) & Theorem~\ref{note:thm:forced}; general linear transports, Proposition~\ref{note:prop:gl}\\
Theorem~\ref{thm:carrier} (endpoint carriers) & Theorem~\ref{note:thm:carrier}; frames from a cycle, Lemma~\ref{note:lem:framesynth}\\
Section~\ref{sec:flat} (flat sector, holonomy) & Propositions~\ref{note:prop:flat} and \ref{note:prop:hol}\\
Theorem~\ref{thm:width} (width and rigidity) & Theorem~\ref{note:thm:widthcompact}; finite-group form, Theorem~\ref{note:thm:width}; $SO(d)$, Corollary~\ref{note:cor:so}\\
Theorem~\ref{thm:graded} (readout-dependent width) & Theorem~\ref{note:thm:graded}\\
Eq.~\eqref{eq:frozen} (frozen cover) & Proposition~\ref{note:prop:frozen}\\
Theorem~\ref{thm:ceiling} (ceilings) & Theorem~\ref{note:thm:ceiling} and Corollaries~\ref{note:cor:comm}--\ref{note:cor:widthacc}, Proposition~\ref{note:prop:sig}\\
Theorem~\ref{thm:cert} (certificate) & Theorem~\ref{note:thm:proto}; snapped models, Theorem~\ref{note:thm:cert}; horizon transfer, Lemma~\ref{note:lem:horizon}\\
\bottomrule
\end{tabular}
\end{center}
\subsection{Setting and conventions}

$V=\R^d$ with its standard inner product; $\Od$ its orthogonal group. Entities are indexed by $i=1,\dots,n$; entity $i$ carries a fiber $V_i\cong V$ and its own element $g_i\in\Od$. The symmetry of the theory is the product group $\mathcal G=\Od^{\,n}$ acting on the data attached to entity $i$ through $g_i$ alone. \emph{This is a modeling premise, not a derived fact}: it is the statement that a fresh entity has no content of its own, so that rotating its coordinates alone changes nothing; the paper adopts it as an inductive bias and derives its consequences.

\paragraph{Data.} Entity $i$ carries $m_i\ge0$ vectors $x_i^{(1)},\dots,x_i^{(m_i)}\in V_i$ (its \emph{channels}), collected in $X_i\in\R^{d\times m_i}$; $\mathcal G$ acts by $X_i\mapsto g_iX_i$. \emph{Invariant data} $e$ (labels, scalars) is fixed by $\mathcal G$. A \emph{transport} $T_{ij}\in\Hom(V_j,V_i)$ is a linear map transforming as $T_{ij}\mapsto g_iT_{ij}g_j^{-1}$; a family of transports on the edges of a graph is a connection, and its products around closed loops are holonomies. A function is \emph{invariant} if it is $\mathcal G$-fixed and a $V_i$-valued function is \emph{equivariant} if it transforms as $x_i$ does.

\paragraph{Two facts from invariant theory} (Weyl 1946, Thm.~2.9.A; Villar et al.\ 2021 for the continuous form). (F1) A continuous $\Od$-invariant function of vectors $v_1,\dots,v_k\in V$ under the diagonal action is a continuous function of their Gram matrix $(\ip{v_a}{v_b})_{ab}$, and the Gram matrix separates orbits: two tuples with the same Gram matrix are related by an element of $\Od$. (F2) A continuous equivariant map $F:(v_1,\dots,v_k)\mapsto V$ takes values in $\spn\{v_1,\dots,v_k\}$; the coefficients can be chosen invariant (the minimal-norm solution $c=\Gamma^{+}(\ip{v_a}{F})_a$, with $\Gamma$ the Gram matrix), and are continuous wherever the span has constant dimension. Proof of (F2): the pointwise stabilizer of $W=\spn\{v_a\}$ in $\Od$ acts on $W^{\perp}$ as $O(W^{\perp})$, fixes $F(v)$, and $O(W^\perp)$ fixes only $0$ in $W^\perp$.

\subsection{Relations are supplied, not computed}

\begin{lemma}[Separation]\label{note:lem:sep}
Let $\mathcal G=G_1\times\cdots\times G_n$ act on $X=X_1\times\cdots\times X_n$ with $G_i$ acting on the $i$-th factor only. A function $f$ on $X$ is $\mathcal G$-invariant if and only if it depends on each $x_i$ only through its $G_i$-orbit. For polynomial invariants, $\R[X]^{\mathcal G}=\bigotimes_i\R[X_i]^{G_i}$. In particular, for per-entity frames with vertex data only, every invariant is a function of the per-entity Gram matrices $\Gamma_i=X_i^{\top}X_i$ and the invariant data; no invariant depends jointly on the directions of two entities' channels.
\end{lemma}

\begin{proof}
Invariance under the subgroup $G_i\times\{e\}$ says $f$ is constant on $G_i$-orbits of $x_i$ for every fixed value of the other coordinates, and conversely. The polynomial statement is the standard fact that $\R[X\times Y]=\R[X]\otimes\R[Y]$ and invariants of a product of groups acting factor-wise are tensor products of invariants (expand a finite tensor sum in linearly independent functions of the other factors, then impose invariance one factor at a time). The last sentence is (F1) applied entity by entity.
\end{proof}

This is the precise form of ``an entity with no content has only relations'': with per-entity frames, no function of the entities alone carries any information about their \emph{relative orientation}. Joint functions of per-entity invariants (norms, local Gram matrices, invariant attributes) are still available and can encode relationships of other kinds; what is excluded is information about how one entity's frame sits relative to another's. Whatever information of that kind a model uses must enter through data that transform across entities, that is, through transports (or through frames, Theorem~\ref{note:thm:carrier}). The premise itself is a chosen inductive bias, not a consequence of fresh names.

\subsection{Forced form}

\begin{theorem}[Forced form, orthogonal transports]\label{note:thm:forced}
Let $i\ne j$ be two entities with independent frames, $x_i\in V_i$, $x_j\in V_j$ and $T_{ij}\in\Od$ transforming as $x_i\mapsto g_ix_i$, $x_j\mapsto g_jx_j$, $T_{ij}\mapsto g_iT_{ij}g_j^{-1}$ for all $(g_i,g_j)\in\Od\times\Od$. A continuous score $s(x_i,x_j,T_{ij})$ is invariant if and only if
\[
s=\phi\big(\norm{x_i},\norm{x_j},\ip{x_i}{T_{ij}x_j}\big)
\]
for a continuous $\phi$, and a continuous message $m(x_i,x_j,T_{ij})\in V_i$ is equivariant if and only if
\[
m\in\spn\{x_i,\ T_{ij}x_j\}
\]
with invariant coefficients. Every such $\phi$ and every such span element is realized, so the form is tight.
\end{theorem}

\begin{proof}
Put $u=x_i$ and $v=T_{ij}x_j$, both in $V_i$. We claim the $\Od\times\Od$-orbits of triples $(x_i,x_j,T_{ij})$ with $T_{ij}$ orthogonal are in bijection with the diagonal $\Od$-orbits of pairs $(u,v)$ through $(x_i,x_j,T)\mapsto(x_i,Tx_j)$. The map is equivariant ($g_i$ acts on both $u$ and $v$, $g_j$ acts on neither) and surjective, since $(u,v)$ is the image of $(u,v,I)$. For injectivity on orbits, suppose $(x_i,x_j,T)$ and $(x_i',x_j',T')$ map to pairs in the same diagonal orbit; applying $g_i$ we may assume $x_i=x_i'$ and $Tx_j=T'x_j'$. Then $\norm{x_j}=\norm{Tx_j}=\norm{T'x_j'}=\norm{x_j'}$, so there is $g_j\in\Od$ with $g_jx_j=x_j'$; replacing $(x_j',T')$ by $(g_j^{-1}x_j',T'g_j)=(x_j,T'g_j)$ we may assume $x_j=x_j'$ and $T'x_j=Tx_j$. Now $T^{-1}T'$ fixes $x_j$, so $h:=T^{-1}T'\in\Stab(x_j)$ and $T'=Th=g_iTg_j^{-1}$ with $g_i=I$, $g_j=h^{-1}\in\Stab(x_j)$, which fixes $x_j$: the two triples lie in one orbit. Hence a continuous invariant of the triple is a continuous invariant of $(u,v)$ under the diagonal action, that is, by (F1), a continuous function of $\norm u=\norm{x_i}$, $\norm v=\norm{x_j}$ and $\ip uv=\ip{x_i}{T_{ij}x_j}$. For the message, invariance under $g_j$ (with $g_i=I$) makes $m$ a function of the $g_j$-orbit of $(x_j,T)$ at fixed $x_i$, which by the same computation is parametrized by $v=Tx_j$; so $m=m(u,v)$ is equivariant under the diagonal action and (F2) puts it in $\spn\{u,v\}$. Realizability is immediate.
\end{proof}

\begin{proposition}[General linear transports]\label{note:prop:gl}
Let $T_{ij}\in\Hom(V_j,V_i)$ be an arbitrary linear map with the same transformation law, and write $T=T_{ij}$. The polynomial invariants of $(x_i,x_j,T)$ are generated by
\[
\begin{gathered}
\tr\big((T^{\top}T)^k\big)\quad(1\le k\le d),\\
x_i^{\top}(TT^{\top})^kx_i,\quad x_j^{\top}(T^{\top}T)^kx_j,\quad
x_i^{\top}(TT^{\top})^kTx_j\quad(0\le k<d),
\end{gathered}
\]
so in particular by $\norm{x_i}^2$, $\norm{x_j}^2$ and the pairing $x_i^{\top}Tx_j$ together with their higher-order companions. Continuous invariants are continuous functions of these generators. Every continuous equivariant message $m\in V_i$ lies in $\spn\{(TT^{\top})^kx_i,\ (TT^{\top})^kTx_j:\ 0\le k<d\}$ with invariant coefficients, continuous on strata where the matrix of these displayed covariants has constant rank. The orthogonal case $TT^{\top}=I$ reduces to Theorem~\ref{note:thm:forced}.
\end{proposition}

\begin{proof}[Proof]
Polynomial invariants: by the first fundamental theorem for $\Od\times\Od$ on mixed tensors, invariants are spanned by complete contractions with the two metrics; a $V_i$-index can be contracted only with a $V_i$-index (those of $x_i$ and the first index of $T$), and likewise for $V_j$. A connected contraction is a chain alternating $T$ and $T^{\top}$, closed into a cycle (the traces) or capped at each end by $x_i$ or $x_j$ (the three vector families, the cap $x_i\cdots x_j$ requiring an odd number of factors, whence $(TT^{\top})^kT$); Cayley--Hamilton for $TT^{\top}$ and $T^{\top}T$ bounds the exponents as displayed. Continuous invariants: $\Od\times\Od$ is compact, so polynomial invariants separate orbits and every continuous invariant is a continuous function of a generating set. Equivariant messages: decompose $V_i=\bigoplus_\lambda E_\lambda$ into eigenspaces of $TT^{\top}$. The stabilizer of $T$ in $\Od\times\Od$ contains, for each $\lambda>0$, the pairs $(g,g')$ acting by a common element of $O(E_\lambda)$ on $E_\lambda$ and on the matching right singular subspace, and for $\lambda=0$ the pairs $(g,I)$ with $g\in O(E_0)$. Equivariance under these subgroups and (F2) applied inside each $E_\lambda$ force the $E_\lambda$-component of $m$ into the span of the $E_\lambda$-components of $x_i$ and $Tx_j$ (for $\lambda=0$, of $x_i$ alone, since $Tx_j$ has no $E_0$-component). Hence $m=\sum_\lambda\Pi_\lambda(a_\lambda x_i+b_\lambda Tx_j)$ with invariant scalars, and each spectral projector $\Pi_\lambda$ is a polynomial of degree $<d$ in $TT^{\top}$ (Lagrange interpolation over the distinct eigenvalues), which gives the displayed span. Let $A$ be the matrix of the displayed covariants. The minimal-norm coefficient vector $A^{+}m$ is invariant under changes of frame and continuous on every stratum where $\rank A$ is constant, because the pseudoinverse is continuous there.
\end{proof}

\subsection{The carrier theorem: what a transport can be computed from}

\begin{theorem}[Carrier]\label{note:thm:carrier}
Let $i\ne j$ be entities with independent frames, let entity $i$ carry channels $X_i\in\R^{d\times m_i}$ and entity $j$ channels $X_j\in\R^{d\times m_j}$, and let $e$ be invariant data. Every continuous map $T(X_i,X_j,e)\in\Hom(V_j,V_i)$ with
\[
T(g_iX_i,\,g_jX_j,\,e)=g_i\,T(X_i,X_j,e)\,g_j^{-1}\qquad\text{for all }g_i,g_j\in\Od
\]
has the form
\[
T=X_i\,C(\Gamma_i,\Gamma_j,e)\,X_j^{\top},\qquad \Gamma_i=X_i^{\top}X_i,\ \Gamma_j=X_j^{\top}X_j,
\]
with $C$ an $m_i\times m_j$ matrix of invariant coefficients, continuous on each stratum where the ranks of $X_i$ and $X_j$ are constant. Consequently:
\begin{enumerate}\renewcommand{\labelenumi}{(\roman{enumi})}\setlength{\itemsep}{0pt}\setlength{\parsep}{0pt}\setlength{\topsep}{0pt}
\item $\rank T\le\min(m_i,m_j,d)$; with one channel per entity, $T=\phi(\norm{x_i},\norm{x_j},e)\,x_ix_j^{\top}$ has rank at most one and is never orthogonal for $d\ge2$ (a single vector can be the endpoint of a computed rank-one transport, not of a full-rank one);
\item if either $m_i=0$ or $m_j=0$ then $T=0$: invariant data alone cannot produce a nonzero covariant transport;
\item $T$ can be orthogonal (or merely full-rank) only if $m_i,m_j\ge d$ and both channel matrices have rank $d$. If $m_i=m_j=d$ and the channel matrices are invertible, the polar factor $F_i=X_i(X_i^{\top}X_i)^{-1/2}\in\Od$ is a continuous equivariant frame ($F(gX)=gF(X)$) and
\[
T=F_i\,\tilde C(\Gamma_i,\Gamma_j,e)\,F_j^{\top},\qquad\tilde C=\Gamma_i^{1/2}C\,\Gamma_j^{1/2}:
\]
an invariant relation matrix read in the two local frames. For rectangular full-rank channel matrices ($m>d$) an equivariant frame can be chosen on each chart of $d$ independent columns, but not necessarily globally (for $d=1$ and two channels the full-rank domain $\R^2\setminus\{0\}$ is connected, so a continuous frame into $O(1)=\{\pm1\}$ would be constant and could not satisfy $F(-X)=-F(X)$); the factorization $T=X_iCX_j^{\top}$ itself needs no frame.
\end{enumerate}
\end{theorem}

\begin{proof}
Let $W_i=\operatorname{col}(X_i)$ and $S_i=\{g\in\Od:\ gX_i=X_i\}$, which acts on $W_i^{\perp}$ as the full group $O(W_i^{\perp})$. Equivariance with $g_j=I$ gives $gT=T$ for every $g\in S_i$, so every column of $T$ is fixed by $O(W_i^{\perp})$ acting on its $W_i^{\perp}$-component; the only fixed vector of $O(W_i^{\perp})$ in $W_i^{\perp}$ is $0$, so the columns lie in $W_i$, i.e.\ $T=P_iT$ with $P_i=X_iX_i^{+}$ the orthogonal projection onto $W_i$. The same argument at $j$ (with $g_i=I$) gives $T=TP_j$. Hence
\[
T=P_iTP_j=X_i\big[X_i^{+}\,T\,(X_j^{+})^{\top}\big]X_j^{\top},
\]
and $C:=X_i^{+}T(X_j^{+})^{\top}$ is invariant ($X_i^{+}\mapsto X_i^{+}g_i^{-1}$ and $T\mapsto g_iTg_j^{-1}$ cancel) and continuous on each fixed-rank stratum, where the pseudoinverse is continuous; by Lemma~\ref{note:lem:sep} and (F1) an invariant of $(X_i,X_j,e)$ is a function of $(\Gamma_i,\Gamma_j,e)$. (i) is the rank of a product; a rank-one matrix has a kernel when $d\ge2$. (ii) is the empty sum. (iii) $T$ of rank $d$ forces $m_i,m_j\ge d$ and full rank; with square invertible channels the polar factor is continuous and equivariant and $X_i=F_i\Gamma_i^{1/2}$; substitute. The obstruction for rectangular channels is the parity argument in the statement.
\end{proof}

\begin{remark}[Transported anchors and holonomy]\label{note:rem:observ}
On a connected graph with inverse-consistent orthogonal transports and anchor channels $X_i$, let $W_i$ be the span of all anchors transported to $i$ along walks; $W_i$ is the fixed space at $i$ of the stabilizer of the supplied data, so every gauge-equivariant vector output at $i$ lies in $W_i$, and with no anchors every such output vanishes. On a flat connection one rank-$r$ anchor supplies rank $r$ at every vertex and further walks add nothing; with holonomy one anchor $x$ at $i$ supplies $\spn\{hx:h\in\mathrm{Hol}_i\}$, the whole fibre exactly when $x$ is cyclic (Lemma~\ref{note:lem:framesynth}). This is an $O(d)$ statement; in $SO(3)$ the cross product $x\times y$ is equivariant and can leave $\spn\{x,y\}$, and in $SO(d)$ $d-1$ channels already determine an oriented frame.
\end{remark}

The endpoint-only restriction is essential: supplied graph transports can bring additional covariant channels to an endpoint, and a sufficiently rich cycle turns one anchor into a full frame.

\begin{lemma}[Frame synthesis]\label{note:lem:framesynth}
Let one entity $a$ carry a single channel $x$ (gauge law $x\mapsto g_ax$) on a connected graph with inverse-consistent orthogonal transports, and let $H$ be a based holonomy at $a$, $H\mapsto g_aHg_a^{-1}$, so that every word $w(H)x$ in $H$ and $H^{\top}$ transforms as a channel of $a$ and $X_a=[x,Hx,\dots,H^{d-1}x]$ transforms as $X_a\mapsto g_aX_a$. Suppose the characteristic polynomial of $H$ is separable (no repeated complex eigenvalue), and let $x$ have nonzero projection on every primary component of $H$ (the excluded set is a finite union of proper subspaces). Then $X_a$ is invertible, and $F_a=\mathrm{GS}(X_a)\in\Od$ (Gram--Schmidt, i.e.\ the $Q$ of the QR factorization with positive diagonal; the polar factor of Theorem~\ref{note:thm:carrier}(iii) serves equally) is a continuous equivariant frame, $F_a\mapsto g_aF_a$, on the open set where $X_a$ is invertible. If instead several based holonomies at $a$ generate a unital algebra $\mathcal A$ acting irreducibly on $V$, then $\mathcal Ax=V$ for \emph{every} $x\ne0$, spanned by words of length at most $d-1$; a continuous equivariant frame is then obtained from any $d$ words whose images are independent, on the open set where they are (the chart caveat of Theorem~\ref{note:thm:carrier}(iii) applies to the choice of words, not to the span).
\end{lemma}
\begin{proof}[Proof sketch]
$H$ is normal, and separability makes each of its primary components a line (eigenvalue $\pm1$) or a rotation plane; the primary projections are polynomials in $H$, so the Krylov space $\R[H]x$ contains the component of $x$ on each, and on a line $\R[H]$ acts by scalars while on a plane it acts as $\mathbb C$, so every nonzero component spans its own subspace and the $d$ Krylov columns are independent. Gram--Schmidt at full rank is continuous and commutes with orthogonal $g$ by uniqueness of the QR factorization, which with $X_a\mapsto g_aX_a$ is the equivariance. In the irreducible case $\mathcal Ax$ is a nonzero invariant subspace, hence $V$, reached by words of length at most $d-1$.
\end{proof}

\subsection{Flat limit and holonomy}

\begin{proposition}[Flat limit]\label{note:prop:flat}
For a connected entity graph with orthogonal, inverse-consistent transports ($T_{ji}=T_{ij}^{-1}$) the following are equivalent: the holonomy of every cycle is the identity; $T_{ij}=h_ih_j^{-1}$ for frames $h\in\Od^{\,n}$; there is a gauge in which every $T_{ij}=I$. In that gauge the forced score is $\phi(\norm{x_i},\norm{x_j},\ip{x_i}{x_j})$ and the forced message $\alpha x_i+\beta x_j$: attention in one global frame. A standard attention head lives in this sector: its transports are all $I$, and its parameter tensors $W_Q^{\top}W_K$ (a bilinear form on the common fiber) and $W_OW_V$ (an endomorphism of it) are not transports and carry no holonomy; they are covariant under a global rotation $g$ only if they co-transform, $W\mapsto gWg^{-1}$, so the parameter tensors of a fixed head are equivariant exactly under their common centralizer, which always contains $\pm I$ and is generically no larger (the head as a function may have more symmetry in degenerate cases, e.g.\ a zero value map). Depth composes these endomorphisms; a standard stack carries no explicitly parameterized connection. This statement alone does not bound a standard stack's ability to compute relational composition by other means. Convention: we call a graph connection \emph{flat} when every cycle holonomy is trivial, i.e.\ when it is pure gauge; in the local-system terminology every graph connection is flat (there are no $2$-cells) and may carry nontrivial holonomy, and we use the stronger sense throughout.
\end{proposition}

\begin{proof}
The equivalences are the standard characterization of a flat connection on a graph: choose a spanning tree, define $h_i$ as the transport from a root along the tree, and note that trivial cycle holonomy makes $h_i$ path-independent and $T_{ij}=h_ih_j^{-1}$; conversely a pure-gauge connection has trivial holonomy, and the gauge $g_i=h_i^{-1}$ makes every transport $I$. With $T=I$, Theorem~\ref{note:thm:forced} reduces as stated. For a standard head the score $\ip{W_Qx_i}{W_Kx_j}$ and message $W_OW_Vx_j$ are functions of two vectors in one fiber; under $x\mapsto gx$ for all entities they transform correctly only if $W_Q^{\top}W_K\mapsto gW_Q^{\top}W_Kg^{-1}$ and $W_OW_V\mapsto gW_OW_Vg^{-1}$, which for fixed matrices holds exactly for $g$ in their common centralizer.
\end{proof}

\begin{proposition}[Holonomy on a graph; classical]\label{note:prop:hol}
Let $\Gamma$ be connected with first Betti number $b_1$, $\mathcal A=\Od^{E}$ the connections (inverse-consistent) and $\mathcal G=\Od^{V}$ the gauge group. (i) Setting the transports of a spanning tree to $I$ is a complete gauge fixing up to one global rotation; the co-tree transports are the holonomies of the fundamental cycles and $\mathcal A/\mathcal G\cong\Hom(F_{b_1},\Od)/\Od$ (simultaneous conjugation). (ii) A connection is flat iff its $b_1$ fundamental holonomies are trivial iff it is pure gauge. (iii) The gauge-invariant functions of a connection are the functions of the conjugacy class of the tuple of fundamental holonomies, which by Procesi's theorem is separated by the traces of \emph{words} in the holonomies and their transposes (traces of powers of single holonomies do not suffice).
\end{proposition}

\begin{proof}[Proof]
(i)--(ii) are Gao, Brodzki and Mukherjee (2021) and the standard tree-gauge argument; the residual freedom is the global rotation because the root frame is free. (iii) is Procesi (1976).
\end{proof}

\subsection{Width law for convex-cell readouts}

\begin{definition}[Composition model]\label{note:def:model}
A finite group $G$ is generated by letters $a\in\Sigma$ through $a\mapsto g_a$. A \emph{transport model} assigns $a\mapsto R_a\in\Od$ and an initial state $x_0\in V$; after the word $w=a_1\cdots a_L$ its state is $R_wx_0$ with $R_w=R_{a_L}\cdots R_{a_1}$, and the target is $g_w=g_{a_L}\cdots g_{a_1}$. $H=\langle R_a\rangle\le\Od$ is the \emph{latent group}; it is assumed finite in Theorems~\ref{note:thm:width} and \ref{note:thm:graded}, and not in Theorem~\ref{note:thm:widthcompact}. A \emph{readout} is a map $r:V\to G$; it is \emph{convex-cell} on a set $S\subseteq V$ if there are pairwise disjoint convex sets $C_g\subseteq V$ ($g\in G$) with $r\equiv g$ on $C_g$ and $S\subseteq\bigcup_gC_g$. Linear argmax readouts $r(x)=\arg\max_g(\ip{w_g}{x}+b_g)$ (with $C_g$ the open cells, so that no point of $S$ is a tie), nearest-prototype readouts (Voronoi cells) and softmax-argmax readouts are convex-cell; multilayer or polynomial readouts need not be. The model is \emph{exact} if $r(R_wx_0)=g_w$ for every word of every length.
\end{definition}

\begin{theorem}[Width law]\label{note:thm:width}
Let a transport model with finite latent group $H$ and a readout that is convex-cell on the orbit $\{R_wx_0\}$ be exact for $G$. Then:
\begin{enumerate}\renewcommand{\labelenumi}{(\roman{enumi})}\setlength{\itemsep}{0pt}\setlength{\parsep}{0pt}\setlength{\topsep}{0pt}
\item $\varphi(R_w):=g_w$ is a well-defined surjective homomorphism $\varphi:H\to G$; let $N=\ker\varphi$.
\item Let $V_N=\Fix(N)$ and $P$ the orthogonal projection onto it. Then $V_N$ is $H$-invariant, $H$ acts on $V_N$ through $G=H/N$ by a representation $\bar\rho$, the readout satisfies $r(Px)=r(x)$ at every orbit point $x=R_wx_0$, and $\bar\rho$ is faithful with $Px_0$ of trivial stabilizer.
\item $d\ge\dim V_N\ge\dmin(G)$, where $\dmin(G)$ is the least dimension of a faithful real representation of $G$.
\item If $d=\dmin(G)$ then $N=1$, $\varphi$ is an isomorphism, $a\mapsto R_a$ is a faithful representation of $G$, and $x_0$ has trivial stabilizer.
\end{enumerate}
Conversely, a faithful representation of dimension $\dmin(G)$ with a generic $x_0$ and the nearest-prototype readout on the orbit $\{\rho(g)x_0\}$ is exact.
\end{theorem}

\begin{proof}
(i) The set $M=\{R_w\}$ of realized transports is a submonoid of the finite group $H$; a finite submonoid of a group is a subgroup (each element has finite order, so its inverse is a positive power), and $M$ contains the generators, so $M=H$: every element of $H$ is realized by some word. If $R_w=R_{w'}$ then the states coincide, so $r$ returns the same label, and exactness gives $g_w=g_{w'}$; thus $\varphi$ is well defined on $H$. For concatenation $ww'$ ($w$ first) we have $R_{ww'}=R_{w'}R_w$ and $g_{ww'}=g_{w'}g_w$, so $\varphi(R_{w'}R_w)=\varphi(R_{w'})\varphi(R_w)$: a homomorphism, surjective because the $g_a$ generate $G$.

(ii) $N$ is normal, so for $h\in H$, $hV_N=\Fix(hNh^{-1})=V_N$, and $hP=Ph$ because $P=\frac1{|N|}\sum_{n\in N}n$ (the averaging operator of a finite orthogonal group is the orthogonal projection onto its fixed space) commutes with $h$. Since $N$ acts trivially on $V_N$, $h|_{V_N}$ depends only on $\varphi(h)$, which defines $\bar\rho$. Now let $x=hx_0$ be an orbit point and $n\in N$. The point $nhx_0$ is realized by some word $w''$ with $R_{w''}=nh$ and $g_{w''}=\varphi(nh)=\varphi(h)=g_w$, so exactness gives $r(nx)=g_w$. Each $nx$ is an orbit point, hence lies in some $C_{g'}$ on which $r\equiv g'$, and disjointness forces $g'=g_w$: the whole $N$-orbit of $x$ lies in $C_{g_w}$. That set is convex, so it contains the average $Px=\frac1{|N|}\sum_nnx$, i.e.\ $r(Px)=g_w=r(x)$. Finally, if $\bar\rho(g)Px_0=\bar\rho(g')Px_0$ with lifts $h,h'$, then $P(hx_0)=P(h'x_0)$, so $r(hx_0)=r(P hx_0)=r(Ph'x_0)=r(h'x_0)$ and $g=g'$ by exactness; hence the $G$-orbit of $Px_0$ has $|G|$ points, its stabilizer is trivial, and $\bar\rho$ is faithful.

(iii) A faithful real representation of $G$ on $V_N$ has dimension at least $\dmin(G)$, and $V_N\subseteq V$.

(iv) If $d=\dmin(G)$ then $V_N=V$, so every $n\in N$ acts as the identity on $V$, i.e.\ $N=\{I\}$; the rest follows from (i)--(ii). The converse is the definition of the nearest-prototype readout, whose cells are convex.
\end{proof}

\begin{remark}[Finiteness and convexity]
Take $G=C_2$, $x_0=(1,0)$, $R_a$ a rotation by $60^\circ$ in $\R^2$ and $r(x,y)=\operatorname{sign}(x^3-3xy^2)$: this decodes parity at every length although $R_a^2\ne I$, with latent group $C_6$ and a readout whose regions are alternating sectors, not convex cells. No convex-cell readout decodes this planar action, since each label's $C_3$ kernel orbit averages to the origin. Above the minimum width a convex-cell readout can decode a cover ($R_a=\mathrm{diag}(-1,\mathrm{Rot}(60^\circ))$ on $\R^3$, sign of the first coordinate), and an infinite latent group is compatible with exactness when its infinite part acts on directions the readout ignores ($\theta/2\pi$ irrational in the same construction). So for convex-cell readouts finiteness at $\dmin$ is a conclusion (Theorem~\ref{note:thm:widthcompact}), above $\dmin$ it can fail, and for arbitrary readouts it must be assumed (Definition~\ref{note:def:graded}). The certificate of Theorem~\ref{note:thm:cert} certifies the snapped model; it does not test whether the learned group is finite.
\end{remark}

\begin{corollary}[The $SO(d)$ version]\label{note:cor:so}
Let $m=\dmin^{O}(G)$ and $s=\dmin^{SO}(G)$ be the least dimensions of faithful representations of $G$ into $\Od$ and into $SO(d)$. Then $m\le s\le m+1$ (adjoin the determinant character to a faithful orthogonal representation). If the transport model of Theorem~\ref{note:thm:width} has $H\le SO(d)$, is exact with a convex-cell readout, and $d=s$, then $N=1$ and $a\mapsto R_a$ is a faithful $SO(d)$-representation of $G$; and no such model exists at $d<s$.
\end{corollary}

\begin{proof}
$\rho\oplus\det\rho$ is faithful and special orthogonal, giving $s\le m+1$; $SO(d)\le\Od$ gives $m\le s$. At $d=s$, Theorem~\ref{note:thm:width}(iii) gives $\dim\Fix(N)\ge m\ge d-1$; an element of $SO(d)$ fixing a subspace of codimension at most one pointwise is the identity (a reflection has determinant $-1$), so $N=1$. If a model existed at $d<s$, then $d=m<s$ would be forced by (iii), and Theorem~\ref{note:thm:width}(iv) would give a faithful $SO(m)$-representation, contradicting the definition of $s$. For skew-exponential parameterizations the relevant minimum is $s$: for example, $D_4$ has $m=2$ and $s=3$.
\end{proof}

\begin{theorem}[Width law without the finiteness hypothesis]\label{note:thm:widthcompact}
Let $G$ be finite and generated by the $g_a$; let $R_a\in\Od$ be \emph{arbitrary} orthogonal matrices, $x_0$ fixed, and let the readout $r$ be exact on every positive word and convex-cell on the reachable set $\{R_wx_0:w\in\Sigma^{*}\}$ in the sense of Definition~\ref{note:def:model}: there are pairwise disjoint convex sets $C_g$ with $r\equiv g$ on $C_g$, so that every reachable state of label $g$ lies in $C_g$. The $C_g$ need be neither open nor closed, and an argmax readout with a fixed tie-breaking rule qualifies whether or not the orbit meets ties (its cells are intersections of open and closed half-spaces). No finiteness, finite-order or margin assumption is made. Let $H=\overline{\langle R_a\rangle}\le\Od$ be the closure of the latent group. Then there is a continuous surjective homomorphism $\varphi:H\to G$ with $\varphi(R_a)=g_a$, and with $N=\ker\varphi$ (compact, normal) and $P=\int_Nn\,dn$ the orthogonal projection onto $\Fix(N)$:
\begin{enumerate}\renewcommand{\labelenumi}{(\roman{enumi})}\setlength{\itemsep}{0pt}\setlength{\parsep}{0pt}\setlength{\topsep}{0pt}
\item $Phx_0\in C_{\varphi(h)}$ for every $h\in H$: projection preserves every reachable label;
\item $G$ acts faithfully on $\Fix(N)$ and $Px_0$ has trivial stabilizer;
\item $d\ge\dim\Fix(N)\ge\dmin(G)$;
\item at $d=\dmin(G)$, $N=\{I\}$, $H\cong G$ is finite, and the learned generators are an exact faithful representation. The $SO(d)$ statement of Corollary~\ref{note:cor:so} holds verbatim with $s=\dmin^{SO}(G)$.
\end{enumerate}
The hypothesis is exactness on all positive words. Finite evaluation does not imply it: a rotation by $2\pi/3+\epsilon$ at width $\dmin(C_3)=2$, with a nearest-prototype readout for the exact $C_3$ orbit, remains correct until the accumulated angular drift crosses a cell boundary. Arbitrarily small nonzero irrational perturbations therefore pass arbitrarily long finite tests. The compact-group argument also applies after lifting to an equivariant polynomial feature space, equipped with an invariant inner product. It does not extend to arbitrary readouts with infinite latent groups: choose free generators of a free subgroup of $SO(3)$ and $x_0$ outside their countably many nontrivial rotation axes. Distinct positive words then give distinct states, so an unrestricted lookup can decode any finite target group at width three. The finite-latent-group hypothesis in Definition~\ref{note:def:graded} is essential for its arbitrary-readout width law.
\end{theorem}

\begin{proof}
Two finite-dimensional convexity facts. (F$_1$) If $D$ is dense in a compact set $S$ then $\mathrm{ri}\,\mathrm{conv}(S)\subseteq\mathrm{conv}(D)$: $\overline{\mathrm{conv}(D)}=\mathrm{conv}(S)$ (one inclusion since $\mathrm{conv}(S)$ is compact, the other by approximating the points of a finite convex combination), and for a convex set in finite dimensions the relative interior equals that of its closure and lies inside the set. (F$_2$) The mean of a probability measure with full support on a compact $S$ lies in $\mathrm{ri}\,\mathrm{conv}(S)$: otherwise an affine functional, nonconstant on $\mathrm{aff}(S)$, is $\ge0$ on $S$ and vanishes at the mean, hence vanishes $\mu$-a.e., hence on $\mathrm{supp}\,\mu=S$, a contradiction.

Let $K$ be the closure in $G\times\Od$ of the positive-word pairs $(g_w,R_w)$. A closed submonoid of a compact group is a subgroup (positive powers of an element return arbitrarily close to the identity, so its inverse lies in the closure of its positive powers), so $K$ is a compact group projecting onto $G$ and onto $H$. For $g\in G$ put $K_g=\{h:(g,h)\in K\}$, $S_g=K_gx_0$ and $D_g=\{R_wx_0:g_w=g\}$; since $G$ is discrete, $D_g$ is dense in $S_g$, and $D_g\subseteq C_g$. Suppose $(g,I)\in K$ for some $g\ne e$. Then $K_g=K_e$, a compact subgroup of $H$; the pushforward of its Haar measure to $S_e=S_g$ has full support, and by (F$_2$) and (F$_1$) its mean lies in $\mathrm{conv}(D_e)\cap\mathrm{conv}(D_g)\subseteq C_e\cap C_g$, contradicting disjointness. Hence $K\to H$ is a continuous bijection of a compact space onto a Hausdorff one, a homeomorphism, and $\varphi=\pi_G\circ\pi_H^{-1}$ is a continuous surjective homomorphism. $N$ is compact and normal, $P$ is the orthogonal projection onto $\Fix(N)$ and commutes with $H$. The fibre $K_{\varphi(h)}$ is the coset $hN$; the pushforward of Haar measure on $hN$ to $S_{\varphi(h)}=hNx_0$ has full support, and its mean $hPx_0=Phx_0$ lies in $\mathrm{conv}(D_{\varphi(h)})\subseteq C_{\varphi(h)}$ by the two facts, whether or not $C_{\varphi(h)}$ is closed: this is (i). The $H$-action on $\Fix(N)$ factors through $G$; if $Phx_0=Ph'x_0$ then $C_{\varphi(h)}\cap C_{\varphi(h')}\ne\emptyset$, so $\varphi(h)=\varphi(h')$: the quotient orbit is free and the action faithful, which is (ii); (iii) follows, and at equality $\Fix(N)=V$ forces $N=\{I\}$, so $\varphi$ is an isomorphism of the compact group $H$ onto the finite group $G$, which is (iv). The $SO$ case is Corollary~\ref{note:cor:so}'s argument applied to this $N$.
\end{proof}

\begin{lemma}[Only the cyclic subspace matters]\label{note:lem:cyclic}
Let $W=\spn\{R_wx_0: w\in\Sigma^{*}\}$, the smallest subspace containing $x_0$ and invariant under all $R_a$. Exactness, the latent group's action on the orbit, and every readout value on the orbit depend only on the restrictions $R_a|_W$.
\end{lemma}

\begin{proof}
All orbit points lie in $W$ and $W$ is invariant, so the dynamics never leave it.
\end{proof}

\subsection{Readout-graded width}

Convex decision cells force the width in Theorem~\ref{note:thm:widthcompact}. A richer readout can instead decode a cover of the target group, sometimes in a smaller state space.

\begin{definition}\label{note:def:graded}
For a finite group $G$ let $\dmin(G)$ be the least dimension of a faithful real representation; let $d_{\mathrm{proj}}(G)=\min\{d:\ G\hookrightarrow PO(d)\}$ with $PO(d)=\Od/\{\pm I\}$; and let $d_{\mathrm{cov}}(G)=\min\{d:\ \exists\,\text{finite }H\le\Od\text{ with }H\twoheadrightarrow G\}$. For a class $\mathcal R$ of readouts let $w_{\mathcal R}(G)$ be the least width at which some transport model with finite latent group and a readout in $\mathcal R$ is exact. Order the classes by inclusion: linear argmax $\subset$ degree-two polynomial features followed by argmax $\subset$ arbitrary functions of the state.
\end{definition}

\begin{theorem}[Width depends on the readout]\label{note:thm:graded}
\leavevmode
\begin{enumerate}\renewcommand{\labelenumi}{(\roman{enumi})}\setlength{\itemsep}{0pt}\setlength{\parsep}{0pt}\setlength{\topsep}{0pt}
\item $w_{\mathrm{lin}}(G)=\dmin(G)$.
\item $w_{\mathrm{quad}}(G)\le\min\big(\dmin(G),\,d_{\mathrm{proj}}(G)\big)$.
\item $w_{\mathrm{any}}(G)=d_{\mathrm{cov}}(G)$.
\item Hence $w_{\mathrm{any}}\le w_{\mathrm{quad}}\le w_{\mathrm{lin}}$, and the inequalities can be strict: for $G=A_5\times A_5$ (and for $A_4\times A_4$), $w_{\mathrm{any}}=w_{\mathrm{quad}}=4<6=w_{\mathrm{lin}}$.
\end{enumerate}
\end{theorem}

\begin{proof}
(i) is Theorem~\ref{note:thm:width} and its converse. (ii) The linear construction gives $\dmin$. For $d_{\mathrm{proj}}$, let $G\hookrightarrow PO(d)$ and let $H\le O(d)$ be its finite preimage, with $H/\{\pm I\}\cong G$. Choose generator lifts and a generic $x_0$. The feature $R_wx_0x_0^{\top}R_w^{\top}$ depends only on $g_w$. It is injective on $G$: equality for lifts $h,h'$ implies $h^{-1}h'x_0=\pm x_0$, which for generic $x_0$ occurs only when $h^{-1}h'=\pm I$. Indeed, each other group element fixes or negates only a proper subspace. The finitely many distinct features have equal Frobenius norm, so nearest-prototype decoding is linear argmax in the quadratic features. (iii) A generic vector has trivial stabilizer in a finite $H\le O(d)$, so a lookup on its orbit can decode any quotient $H\twoheadrightarrow G$. Conversely, exactness alone gives this surjection by the argument of Theorem~\ref{note:thm:width}(i). (iv) The inclusions of readout classes give the inequalities. Since $PSO(4)=SO(4)/\{\pm I\}\cong SO(3)\times SO(3)$, both $A_5\times A_5$ and $A_4\times A_4$ embed in $PO(4)$, yielding the upper bound four. Their preimages in $SO(4)$ have orders $7200$ and $288$, respectively. An involution in either individual factor lifts to a matrix squaring to $-I$. Each target has $\dmin=6$: for $A_5\times A_5$, a tensor product nontrivial on both factors has dimension at least nine, whereas two separate three-dimensional representations attain six. For $A_4\times A_4$, detecting each factor's $V_4$ subgroup likewise requires its three-dimensional irreducible, or a tensor product of dimension at least nine. Finally, every quotient of a finite subgroup of $O(3)$ embeds in $O(3)$. To see this, use $O(3)\cong SO(3)\times C_2$. A finite subgroup projects to a finite rotation group $K$; it is isomorphic either to $K$ or to $K\times C_2$, according as it excludes or contains the central $-I$. The rotation groups are $C_n,D_n,A_4,S_4,A_5$, and their quotients are again rotation groups. A quotient of $K\times C_2$ is either a quotient $A$ of $K$ or $A\times C_2$, depending on whether the central involution's image lies inside the image of $K$. Both embed in $O(3)$. Thus neither target of faithful dimension six is a quotient of a finite subgroup of $O(3)$, proving the lower bound four.
\end{proof}

\begin{remark}[A caveat and a control]
The minimum projective width need not be below the faithful width: $C_2$ has $\dmin=1$ but $d_{\mathrm{proj}}=2$, since $PO(1)$ is trivial. This is why Theorem~\ref{note:thm:graded}(ii) takes the minimum of the linear and projective constructions. A necessary characterization of degree-two width is not claimed. At width four on $A_4\times A_4$, a linear readout cannot be exact while a readout on $xx^{\top}$ can; the existence theorem does not guarantee that optimization finds the latter.

Three consequences sharpen the comparison. (a) The additive gap is unbounded. For $G=A_5^{2m}$, every real irreducible is a tensor product of real irreducibles of the factors. A constituent nontrivial on $j$ factors has dimension at least $3^j\ge3j$, so faithfulness on all $2m$ factors requires dimension at least $6m$, attained by $3^{\oplus2m}$. A block sum of $m$ four-dimensional spin-cover constructions yields a finite $H\le SO(4m)$ surjecting onto $G$, with kernel the blockwise signs $\{\pm1\}^m$. The block features $(x_kx_k^{\top})_{k=1}^m$ separate its labelled orbits, so $w_{\mathrm{quad}}(G)\le4m<6m=w_{\mathrm{lin}}(G)$. This is an unbounded additive gap, not an unbounded ratio; for $m>1$ the kernel need not be the global signs. (b) If $\dmin(G)\le4$, the three widths coincide. The quotient-closure argument for finite $O(3)$ subgroups in the proof, and its cyclic/dihedral analogues in dimensions one and two, imply $d_{\mathrm{cov}}(G)\le j\Rightarrow\dmin(G)\le j$ for $j\le3$. Hence no smaller cover width is possible when $\dmin\le4$. The group $C_3\times A_4$ gives a small example with faithful width five and quadratic width four. (c) Minimum projective dimension is not a threshold at every larger dimension: $PO(d)\cong SO(d)$ for odd $d$, and $A_4\times A_4$ embeds in $PO(4)$ and $PO(6)$ but not in $PO(5)$. Exact decoder realizability is nevertheless monotone in width, since any realization can be padded by unused fixed directions.
\end{remark}

\begin{proposition}[Cross-entropy on a frozen cover]\label{note:prop:frozen}
Let a finite or compact orthogonal group $N$ preserve labels, let $P$ be its Haar-average operator, and let data be sampled uniformly over $N$ within each labelled orbit. For an affine-softmax classifier with logits $Wx+b$ and cross-entropy $\ell$,
\[
\E_{x,y,n}\,\ell(Wnx+b,y)\;\ge\;\E_{x,y}\,\ell(WPx+b,y),
\]
so restricting the classifier to the kernel-fixed projection $Px$ leaves the infimum cross-entropy unchanged. In particular, on a balanced antipodal cover---$M$ labels, states $\pm z_g$ for label $g$, label and sign uniform---
\[
R(W,b)=\frac1{2M}\sum_g\big[\ell(Wz_g+b,g)+\ell(-Wz_g+b,g)\big]\;\ge\;\log\sum_he^{b_h}-\frac1M\sum_gb_g\;\ge\;\log M,
\]
with equality iff $Wz_g$ is a common-logit shift for every $g$ and all biases are equal: every population minimizer outputs the uniform distribution on every orbit state, and with a fixed tie-breaking rule its accuracy is exactly $1/M$. This does not contradict the antipodal bound $(|G|+1)/(2|G|)$ on achievable affine \emph{argmax accuracy}, which is a different optimization problem; it does not predict joint backbone--head training, which can leave the cover; and a finite-length word distribution need not be balanced across the lifts.
\end{proposition}
\begin{proof}
Cross-entropy is convex in the logit vector. Jensen over $n$ replaces $Wnx+b$ by $WPx+b$, and $WPn=WP$ makes the projected classifier invariant. On the antipodal cover, Jensen applied to each pair lower-bounds its average loss by $\ell(b,g)$. Jensen over the biases then gives $\log\sum_he^{b_h}\ge M^{-1}\sum_gb_g+\log M$, with equality iff all biases agree. The log-sum-exp is strictly convex modulo common-logit shifts, so equality for each antipodal pair holds iff $Wz_g$ is such a shift. Thus equality throughout is attained by constant uniform logits and characterizes every population minimizer's predictions.
\end{proof}

\subsection{Ceilings}

\begin{theorem}[Ceilings, joint-process form]\label{note:thm:ceiling}
Let letters $a\in\Sigma$ carry $g_a\in G$ (finite, generated by the $g_a$) and latent transitions $h_a\in H$ (a finite group generated by the $h_a$), let the latent state after a word be any function of $h_w$, and let the readout be any function of the state. Let $K=\langle(g_a,h_a):a\in\Sigma\rangle\le G\times H$ and $N_G=\{g\in G:(g,e_H)\in K\}$, $N_H=\{h\in H:(e_G,h)\in K\}$; both are normal and $G/N_G\cong H/N_H$ (Goursat). Assume the random walk on $K$ with i.i.d.\ uniform letters is irreducible and aperiodic (a letter with $(g_a,h_a)=(e,e)$ suffices). Then for words of length $L$,
\[
\Pr[\text{readout}=g_w]\ \le\ \frac1{|N_G|}+\delta_L,\qquad \delta_L=\TV\big(\mathrm{Law}(g_w,h_w),\ \mathrm{Unif}(K)\big)\xrightarrow[L\to\infty]{}0\ \text{exponentially},
\]
with equality in the limit for a state that determines the $N_G$-coset of $g_w$ (for instance one that determines $h_w$) and a readout returning a fixed representative of it; a state that discards the coset does worse (a constant state has accuracy $1/|G|$). For a readout that sees $h_w$ exactly, the finite-length optimum is $\sum_{h\in H}\max_{g}\Pr_L[g_w=g,h_w=h]$.
\end{theorem}

\begin{proof}
The pair $(g_w,h_w)$ performs a random walk on the finite group $K$ with i.i.d.\ increments uniform on the generating set, so the uniform distribution on $K$ is stationary, and irreducibility and aperiodicity give convergence in total variation at a geometric rate. Under the uniform law, conditionally on $h_w=h$ the target is uniform on $\{g:(g,h)\in K\}$, which is a coset of $N_G$: if $(g_0,h),(g_1,h)\in K$ then $(g_0^{-1}g_1,e)\in K$, and conversely. A readout that sees at most $h$ (any function of $h_w$) returns one label, which is correct with probability $1/|N_G|$ if it lies in the compatible coset and $0$ otherwise, hence at most $1/|N_G|$. At finite length, the success indicator is a $[0,1]$-valued function of $(g_w,h_w)$, so its expectation under $\mathrm{Law}_L$ differs from its expectation under the uniform law by at most $\delta_L$. Attainment and the finite-length formula are immediate.
\end{proof}

\begin{corollary}[Commutative mechanisms]\label{note:cor:comm}
If the latent transitions commute (as matrices, not merely as affine maps with commuting linear parts) and the assumptions of Theorem~\ref{note:thm:ceiling} hold, then $N_G\supseteq[G,G]$ and the asymptotic ceiling is $1/|[G,G]|$, attained by the abelianization $h_a=[g_a]\in G^{\mathrm{ab}}$ with a representative readout. At every finite length, any mechanism whose state is a function of the letter multiset $\mu(w)$---which includes every commutative mechanism, of finite or infinite closure---has accuracy at most
\[
c_G(L)=\sum_{\mu}\Pr_L[\mu]\,\max_{g}\Pr_L[g_w=g\mid\mu],
\]
computable exactly by integer dynamic programming over multisets, and $c_G(L)\ge1/|[G,G]|$.
\end{corollary}

\begin{proof}
If $H$ is abelian then so is $H/N_H\cong G/N_G$, hence $[G,G]\subseteq N_G$. The abelianization realizes $N_G=[G,G]$. For the finite-length statement, the readout is a function of $\mu$, and the Bayes-optimal function of $\mu$ has accuracy $c_G(L)$; the multiset determines the coset of $[G,G]$, so the maximum over a coset of size $|[G,G]|$ of a distribution supported on it is at least $1/|[G,G]|$.
\end{proof}

\begin{proposition}[Conditional mixing, stated as a hypothesis]\label{note:prop:mix}
$c_G(L)\to1/|[G,G]|$ as $L\to\infty$ if and only if the \emph{expected} conditional total variation $\E_\mu\big[\TV(\mathrm{Law}(g_w\mid\mu),\mathrm{Unif}(\text{coset}))\big]$ tends to zero, equivalently (the variable being bounded) iff the conditional law converges to uniform in probability over the multiset; with $n=|[G,G]|$ and $D_\mu=\max_gp_\mu(g)-1/n$ one has $D_\mu\le\TV\le nD_\mu$ and $c_G(L)-1/n=\E[D_\mu]$. Uniform convergence over all multisets is false (the all-identity multiset stays a point mass at every length, with vanishing probability). Exact dynamic programming through a finite length documents values and trends for the groups in the paper; it does not prove the limit, and we do not claim it as a theorem for all groups and alphabets.
\end{proposition}

\begin{corollary}[Cascades of commutative layers]\label{note:cor:cascade}
Let a mechanism consist of $k$ layers; layer $\ell$ has a finite state set $S_\ell$ on which an abelian group $A_\ell$ acts, with finite image $B_\ell\le\mathrm{Sym}(S_\ell)$ (abelian), and its transition on letter $a$ is $s_\ell\mapsto\alpha_\ell(a,s_1,\dots,s_{\ell-1})\cdot s_\ell$ with $\alpha_\ell\in A_\ell$ depending on the letter and on the lower layers' states. Then the transition group $H$ of the composite embeds in the iterated \emph{permutational} wreath product $B_k\wr(B_{k-1}\wr\cdots\wr B_1)$ (kernels $B_\ell^{\,S_1\times\cdots\times S_{\ell-1}}$), which is solvable of derived length at most $k$; hence, under the assumptions of Theorem~\ref{note:thm:ceiling}, $G/N_G$ has derived length at most $k$, $N_G\supseteq G^{(k)}$, and the asymptotic ceiling is $1/|G^{(k)}|$. It is attained within the model class: a $k$-layer cascade with suitable state sets tracking $G/G^{(k)}$ exists by the tower construction (\citealp{shakerinava2026}, ``if'' direction, applied to the quotient), and a representative readout gives $1/|G^{(k)}|$; this is a statement about the class, not about every fixed choice of layer sizes. The ceiling equals $1$ exactly when $k$ is at least the derived length of $G$, which is their threshold; the ceiling below it is what this corollary adds.
\end{corollary}

\begin{proof}
The composite transition for letter $a$ is the bijection
\[
\Phi_a:(s_1,\dots,s_k)\mapsto\big(\alpha_1(a)s_1,\ \alpha_2(a,s_1)s_2,\ \dots,\ \alpha_k(a,s_1,\dots,s_{k-1})s_k\big)
\]
of $S_1\times\cdots\times S_k$ (a bijection: recover the coordinates from the bottom up). The set $W_k$ of all maps of this triangular form, with arbitrary functions $\alpha_\ell$ of the lower coordinates valued in $B_\ell$, is a group, and it is an extension $1\to B_k^{\,S_1\times\cdots\times S_{k-1}}\to W_k\to W_{k-1}\to1$, i.e.\ the permutational wreath product $B_k\wr W_{k-1}$. An extension of a group of derived length $\le k-1$ by an abelian group has derived length $\le k$ ($W^{(k-1)}$ lies in the abelian kernel, so $W^{(k)}=1$), and by induction $W_k$ is solvable of derived length $\le k$. $H=\langle\Phi_a\rangle\le W_k$ inherits the bound, quotients do not increase derived length, and $G/N_G\cong H/N_H$; a group of derived length $\le k$ has $G^{(k)}=1$, so $G^{(k)}\subseteq N_G$. Attainment as stated.
\end{proof}

\begin{corollary}[Width--accuracy bound, any readout]\label{note:cor:widthacc}
Under the assumptions of Theorem~\ref{note:thm:ceiling} with $H\le\Od$ finite, the limiting accuracy is at most $|Q_d|/|G|$, where $Q_d$ is the largest quotient of $G$ with $\dcov(Q_d)\le d$, i.e.\ the largest quotient that is itself a quotient of a finite subgroup of $\Od$; the bound is attained by a lookup readout on a cover. In particular exactness at width $d$ requires $\dcov(G)\le d$ (Theorem~\ref{note:thm:graded}(iii)).
\end{corollary}
\begin{proof}
$G/N_G\cong H/N_H$ is a quotient of the finite group $H\le\Od$, so $\dcov(G/N_G)\le d$ and $1/|N_G|=|G/N_G|/|G|\le|Q_d|/|G|$. For attainment take a finite $H\le\Od$ with $H\twoheadrightarrow Q_d$, lifts $h_a$ of the images of the $g_a$, a generic $x_0$, and the readout that returns a fixed representative of the coset read off the state.
\end{proof}
Attaining this ceiling does not by itself make a convex-cell readout constant on the kernel's cosets; the partial-accuracy law for convex-cell readouts below exactness is not claimed here.

\begin{proposition}[Signature readouts; exactness threshold proved, ceiling conditional]\label{note:prop:sig}
Let each letter $a$ be a fixed increment $e_a\in\R^{|\Sigma|}$, the increments \emph{linearly independent} (one-hot suffices), so that a word is a piecewise-linear path, and let the readout be any function of the path's level-$k$ signature, $k\ge1$. By Chen's identity the signature of a concatenation is the product of signatures in the truncated tensor algebra, and by Magnus's theorem the induced map from the free group $F$ on $\Sigma$ has kernel exactly $\gamma_{k+1}F$; hence the readout is a function of the class of $w$ in the free nilpotent group $N_k=F/\gamma_{k+1}F$. (i) An exact readout exists if and only if $\gamma_{k+1}G=1$, i.e.\ $G$ is nilpotent of class $\le k$. (ii) The accuracy is at most $\E_w[\max_g\Pr(g_w=g\mid[w])]$, and under the hypothesis that the conditional law of $g_w$ given the class converges (in expectation over the class, as in Proposition~\ref{note:prop:mix}) to the uniform law on its $\gamma_{k+1}G$-coset, this tends to $1/|\gamma_{k+1}G|$. Level $1$ is the multiset case of Corollary~\ref{note:cor:comm}. (Dependent increments can identify distinct letters---two generators of $C_2\times C_2$ encoded by one vector have equal signatures at every level---so independence is a hypothesis, not a convenience.)
\end{proposition}

\begin{proof}
Convention: signatures multiply chronologically, $S(ww')=S(w)\otimes S(w')$, while our target convention is $g_{ww'}=g_{w'}g_w$; the map below is therefore a homomorphism into the opposite group $G^{\mathrm{op}}\cong G$, which has the same nilpotency class, and we write $G$ for it. ``If'': when $\gamma_{k+1}G=1$ the map $F\to G$ factors through $N_k$, so equal signatures give equal targets and a lookup readout is exact. ``Only if'': let $M\le N_k$ be the submonoid generated by the images of the letters. Exact decoding of all positive words says that equal classes in $M$ have equal targets, so $[w]\mapsto g_w$ is a well-defined monoid homomorphism $M\to G$. $M$ is cancellative (it sits in a group). For $u,v\in M$ the right ideals $uM$ and $vM$ intersect: otherwise all $2^r$ products of $r$ factors chosen from $\{u,v\}$ would be distinct (cancel the common prefix of two equal products; the first differing factors would give an element of $uM\cap vM$), yet they lie in a ball of radius $O(r)$ of the word metric of the finitely generated nilpotent group $N_k$, whose growth is polynomial (Bass--Guivarc'h)---a contradiction for large $r$. The same holds for left ideals. So $M$ satisfies the Ore condition, its group of fractions is $\langle M\rangle=N_k$ (the letters generate $N_k$), and the monoid homomorphism extends uniquely to a surjective group homomorphism $N_k\to G$. Hence $G$ is a quotient of the free nilpotent group of class $k$ and $\gamma_{k+1}G=1$. (ii) is the Bayes bound and the stated hypothesis.
\end{proof}

\begin{corollary}[The unwired transformer]
A transformer built from standard permutation-equivariant blocks, with no positional encoding, no causal mask and no other position or order features, followed by a permutation-invariant readout, computes a function of the letter multiset and is bounded by $c_G(L)$ at every length.
\end{corollary}

\subsection{Horizon and certificate}

\begin{lemma}[Horizon lower bound]\label{note:lem:horizon}
Let $U_a\in\Od$ be an exact representation with orbit separation $\Delta=\min_{g\ne g'}\norm{U_gx_0-U_{g'}x_0}$, let $W$ be the cyclic subspace of the learned model (Lemma~\ref{note:lem:cyclic}), and let $R_a$ be orthogonal with $\norm{(R_a-U_a)|_W}_{\mathrm{op}}\le\delta$. Then $\norm{R_wx_0-U_wx_0}\le L\delta\norm{x_0}$ for every word of length $L$; nearest-prototype decoding against the exact orbit is correct whenever $L\delta\norm{x_0}<\Delta/2$; and for a linear readout with logits $f_g(x)=w_g^{\top}x+b_g$, raw margin $\mu=\min_{g}\min_{h\ne g}\big(f_g(U_gx_0)-f_h(U_gx_0)\big)>0$ on the exact orbit and $\Lambda=\max_{g\ne h}\norm{w_g-w_h}$, whenever $L\delta\norm{x_0}\Lambda<\mu$. If $A$ is the readout matrix, then $\Lambda\le\sqrt2\norm{A}_{\mathrm{op}}$.
\end{lemma}

\begin{proof}
Telescope in the direction that keeps the perturbation on learned states:
\[
R_{a_L}\cdots R_{a_1}-U_{a_L}\cdots U_{a_1}=\sum_{l=1}^{L}U_{a_L}\cdots U_{a_{l+1}}\,(R_{a_l}-U_{a_l})\,R_{a_{l-1}}\cdots R_{a_1}.
\]
Applied to $x_0$, each right factor $R_{a_{l-1}}\cdots R_{a_1}x_0$ lies in $W$, where the perturbation is bounded by $\delta$, and the left factors are orthogonal; each term is at most $\delta\norm{x_0}$. For the linear readout, a state error $e$ changes every logit difference by at most $\norm{w_g-w_h}\norm e\le\Lambda\norm e$.
\end{proof}

\begin{theorem}[Certificate: search, then verification]\label{note:thm:cert}
Given learned orthogonal generators $R_a$, an initial state $x_0$ and a readout $r$, restrict to the cyclic subspace $W$ of Lemma~\ref{note:lem:cyclic}. \emph{Search} (any method; not part of the certificate): propose a faithful representation $\rho$ of $G$ on $W$ (there are finitely many up to equivalence) and an orthogonal $Q\in O(W)$, e.g.\ by iterating one-sided Procrustes on the stacked equations $R_aQ\approx Q\rho(g_a)$, which is a heuristic and is not claimed to find a global optimum; set $U_a=Q\rho(g_a)Q^{\top}$. \emph{Verification}: (1) check $r(U_gx_0)=g$ for all $g\in G$; (2) compute $\delta=\max_a\norm{(R_a-U_a)|_W}_{\mathrm{op}}$; (3) verify a common robustness radius $\eta>0$ of the readout at the snapped orbit points: every ball $B(U_gx_0,\eta)$ receives the label $g$---for nearest-prototype decoding $\eta=\Delta/2$; for a linear readout $\eta=\mu/\Lambda$ with the constants of Lemma~\ref{note:lem:horizon}; for a general readout $\eta$ must be established by construction, and a readout with no positive radius (for example, $r(-e_1)=a$ and $r\equiv e$ elsewhere, exact on the snapped $C_2$ orbit, wrong after one step for any rotation by $\pi+\epsilon$) certifies nothing about the learned model. If (1) holds, the snapped model $(U_a,x_0,r)$ is exact for every word of every length. If (1)--(3) hold, the learned model $(R_a,x_0,r)$ is certified correct on every word of length $L<\eta/(\delta\norm{x_0})$. Optimal alignment improves the horizon but is not needed for validity; the certificate is stated for the exact snapped model, and a floating-point implementation is certified only with validated error bounds that include construction and arithmetic error.
\end{theorem}

\begin{proof}
$U_w=Q\rho(g_w)Q^{\top}$ exactly, so the snapped state after any word is one of the $|G|$ points checked in (1). By Lemma~\ref{note:lem:horizon} the learned state is within $L\delta\norm{x_0}<\eta$ of the snapped state, which lies in the correctly labelled ball.
\end{proof}

\begin{theorem}[Prototype certificate for the learned model]\label{note:thm:proto}
Let $R_a\in\Od$ $(a\in\Sigma)$, $x_0$ and $r$ be the learned generators, initial state and readout, and let the Cayley table of $G$ be given. A \emph{prototype automaton} is a finite set $Z\subset V$ with $z_0=x_0\in Z$, a transition map $\delta:Z\times\Sigma\to Z$ and a labelling $\lambda:Z\to G$ with
\[
\lambda(z_0)=e,\qquad \lambda(\delta(z,a))=g_a\,\lambda(z)\quad(z\in Z,\ a\in\Sigma).
\]
Put
\[
\epsilon=\max_{z\in Z,\,a\in\Sigma}\norm{R_az-\delta(z,a)},
\]
and let $\eta>0$ be a radius such that $r\equiv\lambda(z)$ on the open ball $B(z,\eta)$ for every $z\in Z$.
Then every word of length $L$ with $L\epsilon<\eta$ is decoded correctly; if $\epsilon=0<\eta$, the model is exact at every length. Verification requires $|\Sigma||Z|$ transition-residual and label-consistency checks, plus $|Z|$ readout-radius checks; this count does not include prototype search or the numerical cost of each check. Group-indexed prototypes use $\delta(z_g,a)=z_{g_ag}$ and $\lambda(z_g)=g$. Orbit-indexed prototypes instead retain several states per label. A breadth-first search with a merge tolerance can propose such a finite automaton, but need not terminate; verification must inspect its total transition map. The prototypes need not be the true orbit and the latent group need not be finite. A merge tolerance cannot invalidate a fully verified certificate because every actual transition residual is included in $\epsilon$. For linear logits $f_g(x)=w_g^{\top}x+b_g$, a valid radius is $\eta=\min_{z,h\ne\lambda(z)}m_{z,h}/\norm{w_{\lambda(z)}-w_h}$, where $m_{z,h}=f_{\lambda(z)}(z)-f_h(z)>0$; a zero denominator imposes no restriction when its margin is positive. For logit differences $x^{\top}S_{gh}x+c_{gh}$ with symmetric $S_{gh}$, put $\rho_z=\norm z$. A valid radius is $\eta=\min_{z,h\ne\lambda(z)}\{-\rho_z+\sqrt{\rho_z^2+m_{z,h}/\norm{S_{\lambda(z)h}}_{\mathrm{op}}}\}$, again omitting zero-norm pairs with positive margin. Replacing the operator norm by an upper bound, such as the Frobenius norm, remains valid. For upper-triangular quadratic features the Euclidean coefficient-difference norm is another valid upper bound. The exact safe radius for each pair is $\sup\{t:\ m_{z,h}+\min_{\norm y\le t}(y^{\top}Sy+2(Sz)^{\top}y)>0\}$, followed by the minimum over pairs. General degree-two logits with linear terms need their additional linear variation included. An approximate trust-region optimizer alone does not validate a lower bound on this radius. These guarantees concern the real-arithmetic system defined by stored weights. For matrices satisfying $\norm{R_a}_{\mathrm{op}}\le1+\kappa$ with $\kappa\ge0$, replace $L\epsilon$ by $\epsilon\sum_{k<L}(1+\kappa)^k$. Floating-point execution additionally requires validated per-step errors $u_t$ in $e_{t+1}\le(1+\kappa)e_t+\epsilon+u_t$, and a readout error allowance.
\end{theorem}
\begin{proof}
Let $z_w=\delta^{*}(z_0,w)$. By consistency and induction $\lambda(z_w)=g_w$ (the convention $g_{wa}=g_ag_w$ of Definition~\ref{note:def:model}). By orthogonality, for $x=R_wx_0$, $\norm{R_ax-\delta(z_w,a)}\le\norm{R_a(x-z_w)}+\norm{R_az_w-\delta(z_w,a)}\le\norm{x-z_w}+\epsilon$, so $\norm{R_wx_0-z_w}\le|w|\epsilon<\eta$ and $r(R_wx_0)=\lambda(z_w)=g_w$. For the quadratic radius, $\Delta(x)-\Delta(z)=(x-z)^{\top}S(x+z)$ with $\norm{x+z}\le2\rho_z+t$ on $\norm{x-z}\le t$, and $(2\rho_z+t)t\norm S_{\mathrm{op}}<m$ solves to the displayed value; the trust-region radius is the exact supremum of the same condition. With $\norm{R_a}_{\mathrm{op}}\le1+\kappa$ the induction gives $\norm{x-z_w}\le\epsilon\sum_{k<|w|}(1+\kappa)^k$.
\end{proof}

Theorem~\ref{note:thm:cert} proves all-word correctness for a verified snapped model and transfers a finite horizon to the learned model. Theorem~\ref{note:thm:proto} bounds the learned model directly, without representation enumeration or an alignment fit.

\begin{remark}[Orbit-indexed certificates for covers]\label{note:rem:covernumbers}
Several states of a covering dynamics share one label, and one representative per label would produce a residual of $2\norm{x_0}$ at a transition to the opposite lift; an orbit-indexed prototype automaton keeps the lifts separate while checking every transition. Approximate breadth-first closure is a search procedure, not proof that the learned generators have finite order: the certificate needs a finite, total transition map with every label-consistency and residual check completed and a valid readout radius. A finite observed failure horizon and an all-word certified horizon answer different questions, and neither proves exactness of the raw learned dynamics.
\end{remark}

Representations of a finite group in a fixed real dimension have finitely many equivalence classes, which makes the snapping search finite at the level of representation types. The search over orthogonal conjugations remains continuous. The certificate does not require a unique nearest class or a unique alignment, and small numerical relator residuals alone do not prove exact group relations.

% Experimental appendix. Counts and numerical results are transcribed from
% the supplied v2 manuscript; no new training results are introduced here.
\section{Experimental protocols and complete reported results}\label{app:exp}

This appendix separates finite evaluation, algebraic extraction, and guarantees for the original learned recurrence. A synthetic example is a word $w=a_1\cdots a_L$ of group generators, with target $g_{a_L}\cdots g_{a_1}$ and state $R_{a_L}\cdots R_{a_1}x_0$. Unless stated otherwise, \emph{fit} means accuracy at least $0.99$ at $L=10$. \emph{Evaluation success} means accuracy $1.000$ on 2,000 fresh words at each reported test length through $L=1,000$. This finite criterion is distinct from correctness on every word. A \emph{snapped model} is a proposed exact representation or cover extracted from learned matrices; its algebra and readout must be checked separately. A \emph{raw-model certificate} bounds the horizon of the recurrence defined by the stored learned matrices, in real arithmetic.

\subsection{Width ladder and the minimal-width comparison}\label{app:ladder}

\paragraph{Eight-group ladder.}
The ladder uses orthogonal transitions, parameterized by matrix exponentials of skew-symmetric generators with an optional reflection, and training words of length at most ten. Evaluation uses fresh words at lengths up to 1,000. Extraction uses the target group's known representations: it is a group-informed analysis, not discovery without group labels. Table~\ref{tab:ladder_full} retains all eight groups, including cases with weak raw generalization. In particular, the extraction result should not be read as the accuracy of the unmodified learned recurrence.

\begin{table}[htbp]
\centering\small
\begin{tabular}{@{}lrrrr@{}}
\toprule
Group & $|G|$ & $\dmin^O$ & Raw, in distribution & Raw, $L=300$\\
\midrule
$C_{10}$ & 10 & 2 & 1.000 & 1.000\\
$D_4$ & 8 & 2 & 1.000 & 1.000\\
$D_{50}$ & 100 & 2 & 0.898 & 0.549\\
$A_4$ & 12 & 3 & 1.000 & 1.000\\
$S_4$ & 24 & 3 & 1.000 & 0.998--1.000\\
$A_5$ & 60 & 3 & 1.000 & 1.000\\
$Q_8$ & 8 & 4 & 1.000 & 0.746\\
$S_5$ & 120 & 4 & 0.998 & 0.843\\
\bottomrule
\end{tabular}
\caption{All eight reported ladder rows. Widths refer to $O(d)$, whereas the controlled $A_4/D_4$ grid below uses $SO(d)$. These are the recorded row summaries; they are not means over a common number of successful seeds.}
\label{tab:ladder_full}
\end{table}

The extracted models for $C_{10},D_4,A_4,S_4$, and $A_5$ realize the intended faithful representations; the $D_4$ result includes three of three reported seeds, and one $S_4$ run stalls. For $D_{50}$, the snapped model scores $1.000$ through $L=1,000$, with reported extraction residual about $0.18$. For $S_5$, the corresponding residual is $0.018$; an exact-orbit nearest-neighbor readout built from snapped prototypes scores $1.000$ through $L=1,000$. These residuals are extraction diagnostics and must not be substituted for the orbit-transition residual in the raw-model certificate without checking their definitions. Successful optimization is uncommon in some cells: approximately one in 25 attempts for $C_{10}$ and one in 18 for $S_5$; the $D_{50}$ search used 25 attempts.

The $Q_8$ row illustrates why the evidence categories matter. A reported extraction residual is $0.103$, and seven width-four runs remain near $0.50\pm0.02$ through $L=300$, consistent with a quotient solution. A width-four run that passes the length-1,000 evaluation fails a length-2,000 test. Width-three runs exceed $0.5$ at $L=300$ but fall below it by $L=1,000$; their learned generators are within $0.006$ of half-turns. This is evidence of approximate finite-horizon behavior, not an exact representation below the faithful-width bound. A heuristic drift scale is not an all-word certificate unless the full prototype residual and readout-radius conditions are verified.

\paragraph{Controlled $SO(d)$ grid.}
For $A_4$ and $D_4$, the comparison uses $R_a=\exp(A_a-A_a^\top)$ at $d=3$ and $d=6$, an affine readout, and softmax cross-entropy over group labels. Both groups have minimum special-orthogonal width three; $D_4$ has smaller orthogonal width two. The batch size is 128, with AdamW and zero weight decay. Training has three stages, each using a fresh optimizer and a warm-up--cosine learning-rate schedule:
\begin{enumerate}\setlength{\itemsep}{0pt}\setlength{\parsep}{0pt}\setlength{\topsep}{0pt}
\item Anchor: 1,500 steps on lengths 1--2, peak learning rate $3\times10^{-3}$.
\item Fit: 4,000 steps on lengths 1--10, peak learning rate $3\times10^{-3}$.
\item Polish: 3,000 steps on lengths $\{2,5,10,20,40\}$, peak learning rate $5\times10^{-4}$.
\end{enumerate}
The protocol includes ten seeds per cell and a second implementation of the same specification. Among the fitting runs in the reported comparison, $17/17$ at $d=3$ satisfy the finite evaluation criterion, versus $7/18$ at $d=6$. These are conditional counts: they neither establish that every fitting minimal-width model is exact nor imply that wider models must fail. Three nonfitting $D_4$ runs at $d=3$ remain at $0.50$ across the evaluated lengths. The theoretical rigidity result assumes exact decoding on every word; training fit alone is not its premise.

\paragraph{Certificate cohort.}
The certificate cohort is identified separately: seeds 0--5 of each $A_4$ and $D_4$ width-three run in the primary implementation, for twelve models in total. Their extracted orbits contain 12 and 8 labeled states, respectively. The reported raw-model horizons are 71--276 for $A_4$ and 100--222 for $D_4$. On a 200,000-letter random-walk probe, three $A_4$ runs first err between 7,253 and 73,076 letters, while three have no error within the probe; $D_4$ first errors range from 10,541 to 67,191. A censored probe is not a lower bound on correctness for \emph{every} word. These twelve certificate records are not evidence that all seventeen fitting runs in the pooled comparison received identical verification.

\paragraph{The forced score under independent frames.}\label{app:m5}
Six candidate entities carry vectors in $\R^{16}$; a relation is one orthogonal map $R\in O(16)$ applied between them, and the task is to select the candidate related to the query. In the control the inputs share one frame; in the test cell every entity's input is rotated by an independent Haar-random element of $O(16)$, with the relation data transformed consistently. Three scores: a fixed-frame bilinear score, the transported score of Theorem~\ref{thm:forced}, and an MLP score on the concatenated pair. Accuracies (shared frame / independent frames, three seeds, spreads at most $0.021$): bilinear $1.000/0.192$, transported $1.000/1.000$, MLP $1.000/0.159$; chance is $0.167$. The non-transported scores fall to chance exactly when the frames become independent; the transported score is unaffected.

\paragraph{Figures for the minimal-width and certificate cohorts.}
Figure~\ref{fig:minwidth} shows every seed of the scripts of record for $A_4$ and $D_4$ at $d=3$ (seeds $0$--$5$) and $d=6$ (seeds $0$--$2$ of the matched protocol); Figure~\ref{fig:horizons} in the main text plots the certified horizon of every certified run in Table~\ref{tab:evidence} against its first observed error (the minimum over the walks run; the one run with horizon $0$ is omitted; censored walks give lower bounds and enter no statistic).

\begin{figure}[htbp]
\centering
\includegraphics[width=\textwidth]{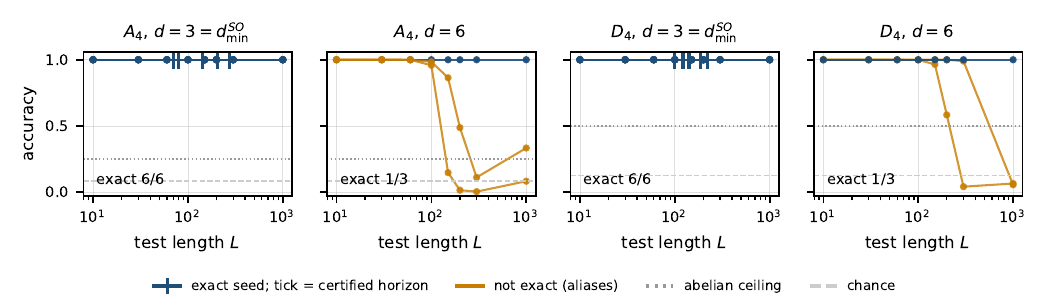}
\caption{Exactness at the minimal width, on the seeds of the scripts of record. Accuracy against test length (training words $\le10$, polish $\le40$) for $A_4$ and $D_4$ at the minimal $SO(d)$ width $d=3$ (seeds $0$--$5$) and at twice it, $d=6$ (seeds $0$--$2$ of the matched protocol). Every $d=3$ seed is exact to length $1000$ (blue), and its snapped generators close to a Latin square of $|G|$ elements with homomorphism defect at most $6\times10^{-13}$ (Table~\ref{tab:evidence}); the tick on each curve is its certified horizon under the record radius ($71$--$276$ for $A_4$, $100$--$222$ for $D_4$), one to two orders of magnitude before the first observed error. At $d=6$ one seed per group is exact and the others fit the training lengths and alias beyond them (orange), some falling through the abelian ceiling (dotted) to chance (dashed).}
\label{fig:minwidth}
\end{figure}

\subsection{Spinor cells and attention-score parity}\label{app:spinor}

\paragraph{Registered protocol.}
The spinor experiments extend the three-stage protocol to $C_3\times A_4$ and $A_4\times A_4$. For the first group, the alphabet is $\{c,c^{-1},a,a^{-1},b,e\}$, where $c$ generates $C_3$, $a,b$ generate $A_4$, and $e$ is the identity. The factors act on disjoint supports in the target permutation representation. Four heads are considered: affine logits; logits affine in $\mathrm{utri}(hh^\top)$; prototype logits $\beta\langle p_g,h\rangle$; and even prototype logits $\beta\langle p_g,h\rangle^2$. The last two are attention-form classification heads. Each registered cell uses twelve seeds; the original prototype cells initialize the score scale at one. All eight registered cells are included in Table~\ref{tab:spinor_full}, including failed controls.

\paragraph{Exploratory protocol.}
After the prototype heads failed to fit, both attention-form variants were changed to cosine scores, with $\log\beta$ initialized at $\log10$ and learned, and all three training stages doubled. Within this exploratory protocol the two width-four heads differ in whether their score is $\beta\cos(p_g,h)$ or $\beta\cos^2(p_g,h)$. The width-four cells use seeds 0--23 and the width-five monotone control seeds 0--11. This changes normalization, initial scale, and training budget relative to the registered protocol; it is not a confirmatory replication of that protocol.

\paragraph{Confirmatory replication.} The three cosine cells were re-run on seeds $24$--$47$, declared in advance, with byte-identical code (sha256 hashes recorded in a declaration written before the first run) against a criterion written down beforehand: at least one exact sign-blind seed in $24$; exactly zero exact monotone seeds at width four; at least three exact monotone seeds per twelve at width five. No run was excluded. Outcomes: sign-blind width four $3/24$ fit and exact (seeds $30$, $36$, $47$), each a certified double cover ($72$ states, kernel $\{\pm I\}$, $R_b^2=-I$, $-x_0$ in the orbit; snapped models exact to $10^5$ with margins $12.1$--$12.7$; certified horizons $10$--$14$ under the record radius; first raw errors at $8{,}362$--$10{,}484$ letters); monotone width four $0/24$, none fitting; monotone width five $7/24$ fit and exact, all faithful. Pooled with the sweep: sign-blind $6/48$, monotone $0/48$ at width four, $13/36$ at width five. Per-run rows are in Table~\ref{tab:evidence}.

\begin{table}[htbp]
\centering\small
\begin{tabular}{@{}llrrl@{}}
\toprule
Group, width & Readout & Fit & Eval. success & Extracted structure\\
\midrule
\multicolumn{5}{@{}l}{\emph{Registered protocol}}\\
$C_3\times A_4$, 4 & quadratic & 2/12 & 2/12 & double cover, 72 states\\
$C_3\times A_4$, 4 & affine & 0/12 & 0/12 & ---\\
$C_3\times A_4$, 4 & monotone prototype & 0/12 & 0/12 & ---\\
$C_3\times A_4$, 4 & even prototype & 0/12 & 0/12 & ---\\
$C_3\times A_4$, 5 & affine & 6/12 & 0/12 & 3 faithful, 36 states\\
$C_3\times A_4$, 5 & quadratic & 2/12 & 0/12 & 2 padded covers, 72 states\\
$A_4\times A_4$, 4 & quadratic & 0/12 & 0/12 & ---\\
$A_4\times A_4$, 6 & affine & 2/12 & 2/12 & faithful, 144 states\\
\midrule
\multicolumn{5}{@{}l}{\emph{Exploratory cosine-score protocol}}\\
$C_3\times A_4$, 4 & even cosine & 3/24 & 3/24 & double cover, 72 states\\
$C_3\times A_4$, 4 & monotone cosine & 0/24 & 0/24 & ---\\
$C_3\times A_4$, 5 & monotone cosine & 7/12 & 6/12 & faithful, 36 states\\
\midrule
\multicolumn{5}{@{}l}{\emph{Confirmatory replication, seeds 24--47}}\\
$C_3\times A_4$, 4 & even cosine & 3/24 & 3/24 & double cover, 72 states\\
$C_3\times A_4$, 4 & monotone cosine & 0/24 & 0/24 & ---\\
$C_3\times A_4$, 5 & monotone cosine & 7/24 & 7/24 & faithful, 36 states\\
\bottomrule
\end{tabular}
\caption{All fourteen spinor cells. Fit and evaluation success concern raw learned matrices. The final column describes the extracted candidate models and orbit sizes; it is not a statement that the raw matrices generate a finite group. Zero finite-evaluation successes are empirical outcomes, not consequences forced by an all-word impossibility theorem.}
\label{tab:spinor_full}
\end{table}

\paragraph{Identifying a cover.}
The two successful registered width-four quadratic runs have 72 approximately closed states over 36 labels, with a lift near $-x_0$. The proposed exact model is the double cover $C_3\times 2T\to C_3\times A_4$, with kernel $\{\pm I\}$. For the target double transposition $b$, the raw residual satisfies $\|R_b^2+I\|_F\le3.6\times10^{-3}$; numerical evaluation after snapping gives residual about $10^{-15}$. The intended exact construction satisfies $R_b^2=-I$. Approximate closure and a small numerical residual identify a candidate structure; neither alone proves the raw latent group is finite.

The raw models pass the finite evaluation criterion through $L=1,000$ and drift at longer horizons. Their original geometric-radius certificates give $L=15$; tighter trust-region radii yield reported horizons 36--38. The snapped models pass the reported evaluations at $10^3$, $10^4$, and $10^5$. All-word correctness of an extracted model requires the finite algebraic and readout verification described below, rather than these long-word tests alone.

At width five, three affine runs yield faithful extracted models despite failing the raw length-1,000 criterion. The two fitting quadratic runs instead yield the width-four cover with an idle fifth direction. For $A_4\times A_4$, none of the twelve width-four quadratic runs fits, so the six-to-four separation remains a constructive expressivity result; the width-six affine control has two evaluation successes. These failures constrain the optimization claim and are retained alongside the positive cases.

\paragraph{Chirality diagnostic.}
In the identified covers, the $C_3$ letters and the involution act in opposite $\mathfrak{su}(2)$ factors of $\mathfrak{so}(4)$; the reported rotation angles are $\pi/3$ and $\pi/2$. A target three-cycle can have components in both factors, with its component in the $C_3$ factor parallel to that generator. This is compatible with composing the intended cover with $(x,y)\mapsto(x\phi(y),y)$, where $\phi:A_4\to C_3$ is the abelianization map. This diagnostic is more specific than an accuracy test, but the universal readout-width theorem does not require every approximate trained solution to have this particular parameterization.

\paragraph{Candidate extraction and finite verification.}
The extraction procedure searches approximate joint closure of labels and latent states, then proposes nearby exact generators. Numerical averaging or rounding can help produce that proposal. A tolerance-based closure search does not establish a finite group, and a near-representation stability theorem does not automatically certify the output of a numerical implementation.

For an all-word certificate of the proposed snapped model, supply a finite labeled orbit $Z$, exact transitions $\widehat R_a z=\delta(z,a)$ for every $z,a$, an initial state in $Z$, and readout correctness on all of $Z$. The transition labels must satisfy $\lambda(\delta(z,a))=g_a\lambda(z)$. These checks imply correctness by induction. Alternatively, verify an exact representation or cover algebraically and enumerate its orbit and readout. Floating-point residuals near machine precision are diagnostic evidence for these identities, not exact arithmetic proofs of them. Search cost and the cost of numerical or symbolic arithmetic are separate from the $|\Sigma||Z|$ transition count.

\paragraph{Certificate for the raw recurrence.}
The prototype search starts from $x_0$, expands by all letters, and merges nearby states at tolerance $0.05\|x_0\|$, requiring label consistency. The output must have a total transition table: every state--letter pair is checked, regardless of how the prototypes were proposed. Let
\[
\epsilon=\max_{z\in Z,a\in\Sigma}\|R_a z-\delta(z,a)\|,
\qquad
r(x)=\lambda(z)\quad\text{whenever }\|x-z\|<\eta.
\]
The certified horizon is $\lceil\eta/\epsilon\rceil-1$ for $\epsilon>0$ and unbounded when $\epsilon=0<\eta$ (Theorem~\ref{thm:cert}); the affine and quadratic readout radii, the near-orthogonal accumulation bound and the real-arithmetic qualification are as stated with Theorem~\ref{note:thm:proto}. A label-consistent approximate orbit can yield a valid finite-horizon certificate without being an exact group; the recorded radii are the conservative ones, and every reported horizon concerns the real-arithmetic recurrence defined by the stored parameters.

\subsection{Readout intervention on fixed cover states}\label{app:frozen}

The intervention holds the dynamics fixed and changes only the decoder. It uses the two registered width-four quadratic covers (seeds 1 and 4) and exploratory even-cosine seed 4. For each run, the 72 states and their 36 labels are enumerated, and a new head is fitted by L-BFGS with five restarts. The experiment is performed separately on the raw approximately closed states and on the snapped cover.

On a balanced \emph{exact} antipodal orbit, each label $g$ has states $\pm z_g$ with equal weight. For affine logits, Jensen's inequality gives the cross-entropy lower bound $\log36$; a minimizing classifier predicts uniformly on these states. This theorem concerns an exact balanced orbit and the optimization objective. It does not assert that every affine classifier has chance accuracy, nor that a raw approximate orbit has exactly the same optimum.

\begin{table}[htbp]
\centering\small
\begin{tabular}{@{}lll@{}}
\toprule
Readout & Reported accuracy & Reported cross-entropy\\
\midrule
Affine, fitted by cross-entropy & near chance ($1/36$) & 3.5835 ($\log36$)\\
Nearest selected lift, affine & 0.500 & not the optimized objective\\
Full quadratic & 1.000 & $\le10^{-5}$\\
Squared prototype & 1.000 & $\le10^{-5}$\\
\bottomrule
\end{tabular}
\caption{Readout intervention with fixed states. The fitted affine loss is reported to four decimals. The exact theorem applies to the balanced antipodal cover; raw approximate states are a separate empirical test. A selected-lift affine head demonstrates achievable one-half accuracy, not optimal affine classification accuracy.}
\label{tab:frozen_full}
\end{table}

The fitted affine heads attain cross-entropy 3.5835 on each reported run. The logit variation across snapped states is of order $10^{-4}$; on raw states it is of order $10^{-2}$. This distinction matters because approximate antipodality does not satisfy the exact theorem's assumptions. The quadratic and squared-prototype heads reach accuracy $1.000$ and cross-entropy at most $10^{-5}$ on the same state sets.

The selected-lift construction chooses one representative per label and assigns its nearest affine prototype; it scores $1/2$. For 36 antipodal label pairs, any convex-cell classifier scores at most $37/72$: at most one decision cell can contain both lifts, since every such cell contains the origin. Thus one-half is a constructive lower bound, not a demonstrated optimum. The experiment distinguishes loss minimization from achievable accuracy while isolating readout expressivity from the task of learning the dynamics.

\subsection{Parsed public benchmarks and controls}\label{app:public}

All models receive the same parsed relational interface within each benchmark. The results therefore concern composition given extracted relations, not text understanding or a comparison with text-level state of the art. Table~\ref{tab:public_full} preserves every reported control. A dash denotes an unreported training-fit summary, not zero accuracy.

\begin{table}[htbp]
\centering\small
\begin{tabular}{@{}llrr@{}}
\toprule
Task & Mechanism & Train fit & Mean OOD\\
\midrule
StepGame & transformer, sinusoidal PE & 0.889 & 0.481\\
 & additive transport, difference readout & --- & 0.753\\
 & additive transport, probe readout & 0.63 & 0.36\\
 & general linear transport, probe readout & --- & 0.73\\
 & rotation transport, probe readout & --- & 0.82\\
 & exact lattice, period 12, probe readout & --- & 0.84\\
 & exact lattice, period 3 & 0.75 & 0.46\\
\midrule
CLUTRR & transformer, sinusoidal PE & 0.942 & 0.436\\
 & unconstrained matrix scan, $d=32$ & --- & 0.848\\
 & orthogonal matrix scan, $d=32$ & 0.98 & 0.202\\
\bottomrule
\end{tabular}
\caption{Complete reported public-benchmark controls. Transformers have three layers and width 64. StepGame trains on $k\le5$ and tests on $k=6$--$10$; CLUTRR trains on $k\in\{2,3\}$ and tests on $k=4$--$10$. The composer and probe rows use different readouts and cannot be treated as consecutive one-variable interventions. The CLUTRR orthogonal control reports three seeds.}
\label{tab:public_full}
\end{table}

\paragraph{StepGame.}
The dataset version includes the maintainers' February 2024 correction. An audit of the parsed interface covers more than 65,000 examples, with 100\% agreement on the clean strata tested; three residual generator defects were identified, and the training-split corruption census is 10.2\%. This is an audit of those strata and that release, not a guarantee for arbitrary textual inputs.

StepGame relations compose as displacements in $\mathbb Z^2$. The additive composer with its difference readout obtains mean OOD accuracy 0.753 versus 0.481 for the transformer, a difference of 27.2 percentage points; the recorded difference is at least 23 points at every evaluated OOD hop count. A gauge-attention aggregation variant matches the mean-aggregation composer to four decimals in the reported experiment. Commutativity alone does not imply that every possible learned attention weighting must equal mean aggregation.

The source-initialized probe controls use a different readout. Additive transports score 0.36, general linear maps 0.73, rotations 0.82, and a period-12 exact lattice representation 0.84. Under this probe construction, a linear map preserves a zero state at an unreached node, whereas an affine translation can create a signal there. The linear-to-rotation comparison is consistent with improved conditioning: the unconstrained matrices have reported singular values between 0.03 and 1.6, whereas rotations preserve norm. These are mechanisms compatible with the controls, rather than proofs that they account for every performance difference.

With lattice period three, the recorded OOD accuracy is approximately 0.457 (rounded to 0.46 in the table), consistent with collisions modulo the period. Period 12 avoids those particular collisions within the benchmark range. A periodic representation of $\mathbb Z^2$ still aliases displacements outside its identifiable range; the result does not establish a faithful finite-dimensional periodic encoding of all displacements.

\paragraph{CLUTRR.}
The experiment uses the Edge Transformer release~\citep{bergen2021}, variant \texttt{db\_9b2173cf}, with 35,394 training examples of lengths $k\in\{2,3\}$. Mixed-length test files are split by $k$ before scoring over $k=2$--$10$. The baseline has sinusoidal positional encoding, three layers, and width 64. The scan uses 18 relation matrices of width 32; the orthogonal control constrains the same family of matrices to $O(32)$ on the same interface.

The unconstrained scan has mean OOD accuracy 0.848 versus 0.436 for the transformer. The orthogonal control has mean 0.202 while fitting its training data at about 0.98. Hence the orthogonality intervention changes the OOD mean by 64.6 percentage points. This large controlled difference supports a task-dependent choice of transition constraints.

There is a relevant algebraic motivation: if $e\ne1$ is an idempotent in a monoid and $\rho$ is a faithful unital representation, then $\rho(e)$ cannot be invertible, since $\rho(e)^2=\rho(e)$ would imply $\rho(e)=I$. The experiment does not establish that CLUTRR's typed composition rules form a particular monoid, nor that a relation such as sibling composition satisfies an unconditional idempotence identity. It therefore does not identify idempotence as the sole cause of the measured effect. The narrower empirical conclusion is that imposing isometries helps one parsed task and substantially harms the other.

\subsection{Evidence tables}\label{app:evidence}
Constructed witnesses (the first evidence field) are in Table~\ref{tab:witnesses}; every certified or snap-certified run, with the four fields---constructed, snapped, raw-certified, finite evaluation---is a row of Table~\ref{tab:evidence}.

\begin{table}[htbp]
\centering\small
\begin{tabular}{@{}llrrl@{}}
\toprule
target $G$ & constructed $H\le O(4)$ & $|H|$ & labels & readout separation\\
\midrule
$C_3\times A_4$ ($\dmin=5$) & $C_3\times2T$, $x\mapsto q_1x\bar q_2$ & 72 & 36 & quadratic (min $0.176$); $R_b^2=-I$\\
$A_4\times A_4$ ($\dmin=6$) & $2T\times2T/\pm(1,1)$ & 288 & 144 & quadratic\\
$A_5\times A_5$ ($\dmin=6$) & $2I\times2I/\pm(1,1)$ & 7200 & 3600 & quadratic\\
\bottomrule
\end{tabular}
\caption{Constructed cover witnesses at width $4$, verified numerically (closure, kernel $\{\pm I\}$, label count, quadratic separation of the antipodal-pair labels).}
\label{tab:witnesses}
\end{table}

{\scriptsize
\setlength{\tabcolsep}{1.6pt}\begin{longtable}{@{}lllcrllrlllll@{}}
\caption{Evidence table, one line per certified or snap-certified run. Run: group (C3A4 $=C_3\times A_4$, A4A4 $=A_4\times A_4$, C3A4$^*$ the 4-letter pilot alphabet), width, readout (lin, quad, pm/pe $=$ prototype attention monotone/sign-blind, cm/ce $=$ cosine attention monotone/sign-blind), seed; protocol (rec $=$ record, exp $=$ exploratory, conf $=$ confirmatory). \emph{Finite evaluation}: fit ($\ge0.99$ at $L=10$) / exact ($1.000$ on $2000$ words at $L=1000$) / raw accuracy at $10^4$. \emph{Raw model}: orbit states (cover index), residual $\epsilon$, radius $\eta$ (record / operator-norm / exact where computed), certified horizon (same three), first observed error (minimum over the walks run; $\dagger$ $=$ none within the walk, a lower bound). \emph{Snapped model}: latent closure (LS $=$ Latin-square closure of the snapped generators), kernel of the cover (ord.\ 2 $=$ order two on a padded block, not $\pm I$), homomorphism defect (floating point), margin $m$ on every orbit state, finite evaluation at $10^3/10^4(/10^5)$.}\label{tab:evidence}\\
\toprule
run & prot. & fit/ex/$10^4$ & orbit & $\epsilon$ & $\eta$ & cert.\ $L$ & first err. & closure & ker. & defect & $m$ & snapped eval\\
\midrule
\endfirsthead
\toprule
run & prot. & fit/ex/$10^4$ & orbit & $\epsilon$ & $\eta$ & cert.\ $L$ & first err. & closure & ker. & defect & $m$ & snapped eval\\
\midrule
\endhead
\bottomrule
\endfoot
A4A4-6-lin-s2 & rec & y/y/1.00 & 144 (1) & 0.1531 & 2.23/2.23/2.23 & 14/14/14 & 14268 & 144 & 1 & 3e-15 & 3.8 & 1/1/1 \\
A4A4-6-lin-s9 & rec & y/y/0.47 & 144 (1) & 0.2560 & 1.84/1.84/1.84 & 7/7/7 & 2546 & 144 & 1 & 3e-15 & 2.6 & 1/1/1 \\
C3A4-4-quad-s1 & rec & y/y/0.00 & 72 (2) & 0.0218 & 0.342/0.471/0.831 & 15/21/38 & 1737 & 72 & $\pm I$ & 3e-13 & 7.7 & 1/1/1 \\
C3A4-4-quad-s4 & rec & y/y/0.00 & 72 (2) & 0.0216 & 0.341/0.479/0.784 & 15/22/36 & 2872 & 72 & $\pm I$ & 3e-13 & 7.6 & 1/1/1 \\
C3A4-5-lin-s1 & rec & y/n/0.05 & 36 (1) & 0.0641 & 2.58/2.58/2.58 & 40/40/40 & 457 & 36 & 1 & 2e-15 & 6.1 & 1/1/1 \\
C3A4-5-lin-s2 & rec & y/n/0.04 & 36 (1) & 0.0393 & 2.45/2.45/2.45 & 62/62/62 & 533 & 36 & 1 & 2e-15 & 6.9 & 1/1/1 \\
C3A4-5-lin-s6 & rec & y/n/0.00 & 36 (1) & 0.0673 & 2.5/2.5/2.5 & 37/37/37 & 572 & 36 & 1 & 2e-15 & 6.1 & 1/1/1 \\
C3A4-5-quad-s3 & rec & y/n/0.05 & 72 (2) & 0.4035 & 0.299 & 0 & 613 & 72 & ord.\ 2 & 3e-15 & 5.5 & 1/1/1 \\
C3A4-5-quad-s7 & rec & y/n/0.28 & 72 (2) & 0.2177 & 0.277 & 1 & 260 & 72 & ord.\ 2 & 2e-15 & 6.7 & 1/1/1 \\
A4-3-lin-s0 & rec & y/y/-- & 12 (1) & 0.0332 & 2.387 & 71 & 7253 & 12 (LS) & -- & 6e-13 & -- & 1/1 \\
A4-3-lin-s1 & rec & y/y/-- & 12 (1) & 0.0099 & 2.736 & 276 & $\dagger$ & 12 (LS) & -- & 3e-15 & -- & 1/1 \\
A4-3-lin-s2 & rec & y/y/-- & 12 (1) & 0.0345 & 2.757 & 79 & 7578 & 12 (LS) & -- & 4e-13 & -- & 1/1 \\
A4-3-lin-s3 & rec & y/y/-- & 12 (1) & 0.0113 & 2.302 & 204 & $\dagger$ & 12 (LS) & -- & 1e-14 & -- & 1/1 \\
A4-3-lin-s4 & rec & y/y/-- & 12 (1) & 0.0191 & 2.728 & 142 & 73076 & 12 (LS) & -- & 2e-14 & -- & 1/1 \\
A4-3-lin-s5 & rec & y/y/-- & 12 (1) & 0.0159 & 2.247 & 141 & $\dagger$ & 12 (LS) & -- & 2e-14 & -- & 1/1 \\
D4-3-lin-s0 & rec & y/y/-- & 8 (1) & 0.0139 & 3.090 & 222 & 61582 & 8 (LS) & -- & 8e-15 & -- & 1/1 \\
D4-3-lin-s1 & rec & y/y/-- & 8 (1) & 0.0178 & 2.530 & 141 & 67191 & 8 (LS) & -- & 3e-14 & -- & 1/1 \\
D4-3-lin-s2 & rec & y/y/-- & 8 (1) & 0.0246 & 3.049 & 123 & 11315 & 8 (LS) & -- & 3e-14 & -- & 1/1 \\
D4-3-lin-s3 & rec & y/y/-- & 8 (1) & 0.0161 & 3.027 & 187 & 43252 & 8 (LS) & -- & 8e-15 & -- & 1/1 \\
D4-3-lin-s4 & rec & y/y/-- & 8 (1) & 0.0273 & 2.732 & 100 & 10541 & 8 (LS) & -- & 7e-14 & -- & 1/1 \\
D4-3-lin-s5 & rec & y/y/-- & 8 (1) & 0.0213 & 2.616 & 122 & 12622 & 8 (LS) & -- & 5e-14 & -- & 1/1 \\
C3A4-4-ce-s4 & exp & y/y/1.00 & 72 (2) & 0.0024 & 0.0312/0.192/0.224 & 12/78/91 & 32532 & 72 & $\pm I$ & 5e-14 & 11.2 & 1/1/1 \\
C3A4-4-ce-s12 & exp & y/y/0.00 & 72 (2) & 0.0038 & 0.0397/0.231/0.259 & 10/60/67 & 2471 & 72 & $\pm I$ & 5e-15 & 10.1 & 1/1/1 \\
C3A4-4-ce-s14 & exp & y/y/0.00 & 72 (2) & 0.0062 & 0.11/0.721/0.83 & 17/115/133 & 2148 & 72 & $\pm I$ & 2e-14 & 9.2 & 1/1/1 \\
C3A4-5-cm-s2 & exp & y/y/0.00 & 36 (1) & 0.0124 & 0.339/1.34/1.34 & 27/108/108 & 4378 & 36 & 1 & 3e-13 & 11.8 & 1/1/1 \\
C3A4-5-cm-s6 & exp & y/y/0.99 & 36 (1) & 0.0029 & 0.181/0.708/0.708 & 62/242/242 & 1687 & 36 & 1 & 3e-14 & 12.4 & 1/1/1 \\
C3A4-5-cm-s8 & exp & y/y/0.00 & 36 (1) & 0.0129 & 0.23/1/1 & 17/77/77 & 3627 & 36 & 1 & 2e-15 & 10.3 & 1/1/1 \\
C3A4-5-cm-s9 & exp & y/y/0.98 & 36 (1) & 0.0110 & 0.103/0.444/0.444 & 9/40/40 & 11129 & 36 & 1 & 2e-15 & 10.5 & 1/1/1 \\
C3A4-5-cm-s10 & exp & y/y/0.00 & 36 (1) & 0.0139 & 0.192/0.869/0.869 & 13/62/62 & 6400 & 36 & 1 & 2e-15 & 12.0 & 1/1/1 \\
C3A4-5-cm-s11 & exp & y/y/0.22 & 36 (1) & 0.0095 & 0.224/1.02/1.02 & 23/107/107 & 1310 & 36 & 1 & 2e-15 & 11.3 & 1/1/1 \\
C3A4-4-ce-s30 & conf & y/y/0.03 & 72 (2) & 0.0129 & 0.131 & 10 & 8362 & 72 & $\pm I$ & 1e-13 & 12.1 & 1/1/1 \\
C3A4-4-ce-s36 & conf & y/y/0.54 & 72 (2) & 0.0037 & 0.056 & 14 & 9484 & 72 & $\pm I$ & 1e-13 & 12.7 & 1/1/1 \\
C3A4-4-ce-s47 & conf & y/y/1.00 & 72 (2) & 0.0075 & 0.093 & 12 & 10484 & 72 & $\pm I$ & 8e-14 & 12.3 & 1/1/1 \\
C3A4-5-cm-s25 & conf & y/y/0.00 & 36 (1) & 0.0049 & 0.092 & 18 & 3640 & 36 & 1 & 3e-13 & 12.5 & 1/1/1 \\
C3A4-5-cm-s28 & conf & y/y/0.00 & 36 (1) & 0.0034 & 0.076 & 21 & 2343 & 36 & 1 & 2e-15 & 12.1 & 1/1/1 \\
C3A4-5-cm-s35 & conf & y/y/0.05 & 36 (1) & 0.0285 & 0.114 & 3 & 6420 & 36 & 1 & 2e-15 & 9.6 & 1/1/1 \\
C3A4-5-cm-s37 & conf & y/y/0.99 & 36 (1) & 0.0052 & 0.142 & 27 & 1308 & 36 & 1 & 2e-13 & 12.2 & 1/1/1 \\
C3A4-5-cm-s41 & conf & y/y/0.65 & 36 (1) & 0.0093 & 0.155 & 16 & 1565 & 36 & 1 & 2e-15 & 12.2 & 1/1/1 \\
C3A4-5-cm-s42 & conf & y/y/0.00 & 36 (1) & 0.0095 & 0.156 & 16 & 3218 & 36 & 1 & 6e-13 & 11.9 & 1/1/1 \\
C3A4-5-cm-s44 & conf & y/y/1.00 & 36 (1) & 0.0082 & 0.099 & 12 & $\dagger$ & 36 & 1 & 2e-15 & 11.0 & 1/1/1 \\
C3A4*-4-quad-s14 & pilot & y/y/0.00 & 72 (2) & 0.0035 & 0.107/0.147/0.245 & 30/42/70 & 2335 & 72 & $\pm I$ & 1.2e-14 & -- & 1/1 \\
C3A4*-4-quad-s19 & pilot & y/y/0.14 & 72 (2) & 0.0032 & 0.101/0.139/0.233 & 31/42/72 & 5825 & 72 & $\pm I$ & 1.6e-14 & -- & 1/1 \\
C3A4*-4-quad-s23 & pilot & y/y/0.39 & 72 (2) & 0.0069 & 0.0919/0.13/0.208 & 13/18/30 & 5511 & 72 & $\pm I$ & 4.7e-13 & -- & 1/1 \\
C3A4*-4-quad-s25 & pilot & y/y/0.00 & 72 (2) & 0.0047 & 0.102/0.132/0.215 & 21/27/45 & 2985 & 72 & $\pm I$ & 5.6e-14 & -- & 1/1 \\
\end{longtable}}

\subsection{Supporting depth, signature, transformer, and graph experiments}\label{app:additional}

\paragraph{Commutative cascades.}
The depth experiment uses stacked diagonal scans whose layer-$(k+1)$ angles depend on the current token and the state of layer $k$. It has four groups, $k=1,2,3$, and two seeds per cell, for 24 learned runs, evaluated at lengths 250--300. In the finite-group cascade model covered by the theory, depth $k$ yields the limiting ceiling $1/|G^{(k)}|$ under the stated joint mixing assumptions. The values are shown in Table~\ref{tab:depth_bounds}; they are structural bounds for that model, not universal claims about every architecture described as a diagonal state-space model.

\begin{table}[htbp]
\centering\small
\begin{tabular}{@{}lrrrl@{}}
\toprule
Group & $k=1$ & $k=2$ & $k=3$ & First exact constructed tower\\
\midrule
$D_4$ & $1/2$ & 1 & 1 & depth 2\\
$A_4$ & $1/4$ & 1 & 1 & depth 2\\
$S_4$ & $1/12$ & $1/4$ & 1 & depth 3\\
$A_5$ & $1/60$ & $1/60$ & $1/60$ & none in this class\\
\bottomrule
\end{tabular}
\caption{Derived-series limiting ceilings for the finite cascade setting. A finite test length requires the corresponding mixing correction. Constructed witnesses and learned scans are evaluated separately.}
\label{tab:depth_bounds}
\end{table}

The learned cells are consistent with the bounds: the recorded $A_5$ accuracy is 0.0168 versus the limiting value $1/60\simeq0.0167$. Constructed towers for $D_4,A_4,S_4$, following the derived-series construction, score 1.0000 at lengths 300 and 1,000 at depths 2, 2, and 3. Learned cascades do not reach these witnesses in the cells whose ceiling is one; the best reported learned accuracy there is 0.323 at $L=300$. A constructive expressivity threshold and a successful optimization result are different claims.

\paragraph{Signature controls.}
Level-one constructions using letter counts realize the abelianization readout: the limiting accuracies are $0.50$ for $D_4$, $0.25$ for $A_4$, $1/12$ for $S_4$, and $0.04$ for $D_{50}$. A level-two $D_4$ witness uses modular arithmetic on integer signature coordinates and scores 1.000 through the reported length-1,000 evaluation. Learned signature heads do not reach these constructed values; for $A_4$, the reported accuracies are 0.068--0.091 across the tested levels, close to chance $1/12$ and below $0.25$.

The exactness obstruction and the partial-accuracy ceiling have different hypotheses. Exact decoding through a level-$k$ nilpotent quotient imposes the corresponding lower-central-series condition. An accuracy ceiling of $1/|\gamma_{k+1}G|$ additionally requires that the representation and readout factor through the stated finite nilpotent quotient, together with the relevant distributional or mixing assumptions. It is not a blanket bound on arbitrary nonlinear functions of real-valued truncated signatures. Normalization also changes the information presented to the learned head; the modular witnesses rely on integer coordinates that a normalized feature representation need not preserve.

\paragraph{Order-blind and causal transformers.}
These experiments use JAX transformers with three layers, width 64, and training words of length at most ten. The $A_4$ runs use two 10,000-step training cycles. The three conditions are an order-blind bag model, a causal model without positional encoding, and a causal model with sinusoidal encoding. Absence of positional embeddings alone does not make a model order blind: a causal mask is itself order information.

For an order-blind model, the exact finite-length bound is
\[
c_G(L)=\sum_{\nu}\max_{g\in G}
\Pr\{\operatorname{counts}(w)=\nu,\ g_w=g\}.
\]
It depends on the alphabet and sampling distribution. For the reported $L=10$ settings, $c_{D_4}(10)=73/128\simeq0.5703$ and $c_{A_4}(10)\simeq0.4420$. The bag accuracies are 0.43 and 0.34--0.37, respectively. Including the identity letter makes an irreducible group walk aperiodic; it does not make its finite-length label distribution exactly uniform.

Causal models exceed these bag ceilings in distribution: reported $D_4$ accuracies are 0.81--0.88, and the $A_4$ mask-only accuracies are 0.84--0.89. The nine reported out-of-range runs fall near chance at the tested longer lengths. The $A_4$ sinusoidal-plus-mask condition is underfit at 0.56--0.65 after 20,000 steps and is retained as such. Length eleven was not evaluated, so these results do not establish failure immediately beyond the training range. They characterize these protocols and do not prove that transformers cannot learn composition or extrapolate by another mechanism.

\paragraph{Knowledge-graph null result.}
The FB15k-237 experiment uses a pretrained TransE model~\citep{bordes2013} and 776 two-hop patterns mined from the training graph. Its registered prediction related a loop-residual norm to Hits@10. The measured Spearman correlations are $+0.055$ for raw evaluation and $-0.039$ for filtered evaluation, so the predicted association is not supported. An exploratory route-divergence analysis gives correlation $+0.345$, with divergence 12.6\% in the least-frustrated quartile and 47.9\% in the most-frustrated quartile. The null and exploratory outcomes concern different quantities.

For translation transports, a tree gauge removes tree-edge offsets and leaves sums of offsets around cycles as residuals. This provides a geometric interpretation of the diagnostic. It does not imply that a residual norm is sufficient to predict link-prediction accuracy, nor does the exploratory correlation establish a causal mechanism.

\section{Reproducibility, use of AI, and ethics}
\paragraph{Reproducibility.} Appendix~\ref{app:proofs} gives full proofs. Appendix~\ref{app:exp} specifies training stages, decoder families, finite evaluation criteria, and the distinction between candidate extraction and certificate verification. Registered and exploratory cells are reported separately, including unsuccessful cells. The real-arithmetic guarantees do not constitute validated floating-point execution bounds.
\paragraph{Statement on the use of generative AI.} In this work we used generative AI tools (large language models) for the following tasks: to help develop the theoretical framework and conceptual model, to formulate mathematical claims, to provide ingredients for and assist in the writing of proofs, to propose and refine hypotheses (including the pre-registered predictions), to design and give feedback on the experimental methodology, to implement methods, and to interpret results. We did not use generative AI tools to generate data: every dataset is sampled by code from the stated distributions. Translation, dataset cleaning and reformatting, and qualitative or thematic analysis are not applicable to this work. Additionally, we used generative AI tools to create and edit code, to create figures, to search and summarize the literature, to brainstorm, to draft and edit the paper, and to format references. We have reviewed all AI-assisted work. Every theorem statement and proof was checked by the authors and, separately, subjected to adversarial review by a model instance independent of the one that wrote it, with executable counterexample checks; the code was read by the authors and run from the released seeds, and the reported numbers reproduce from the released scripts and manifests; every citation was checked against arXiv or publisher records. We take responsibility for the final content of this work, including text, claims and artifacts produced with the aid of generative AI.
\paragraph{Ethics.} The study concerns mathematical models, synthetic group-composition tasks, and public relational benchmarks. It introduces no human-subject experiment or new collection of personal data.
\end{document}